\documentclass{aims}

\usepackage{amsmath,amssymb}
\usepackage{algorithm}
\usepackage{algorithmic}
\usepackage{xcolor}
\usepackage{caption}
\usepackage{subcaption}
\usepackage{float}
\usepackage{setspace}
\usepackage{enumerate}
\usepackage{bm}
\usepackage{multirow}
\usepackage{cite}
\usepackage{microtype}
\usepackage{verbatim}
\usepackage{array}
\usepackage{graphicx}
\newcolumntype{C}[1]{>{\centering\let\newline\\\arraybackslash\hspace{0pt}}m{#1}}

\usepackage[colorlinks=true]{hyperref}
\hypersetup{urlcolor=blue, citecolor=red}
\def\currentvolume{X}
 \def\currentissue{X}
  \def\currentyear{200X}
   \def\currentmonth{XX}
    \def\ppages{X--XX}

\newtheorem{theorem}{Theorem}[section]

\newtheorem{proposition}[theorem]{Proposition}

\theoremstyle{definition}

\DeclareMathOperator*{\argmin}{arg\,min}

\title[QC Model with CNN Features for Image Registration]{Quasiconformal Model with CNN Features for Large Deformation Image Registration}

\author[Ho Law, Gary P. T. Choi, Ka Chun Lam and Lok Ming Lui]{}

\email{hlaw@gatech.edu}
\email{ptchoi@mit.edu}
\email{kachun.lam@nih.gov}
\email{lmlui@math.cuhk.edu.hk}

\thanks{This work was supported in part by the National Science Foundation under Grant No.~DMS-2002103 (to Gary P.~T.~Choi), and HKRGC GRF under project ID 14305919 (to Lok Ming Lui).}



\usepackage{amsopn}

\begin{document}

\maketitle

% Enter the first author's name and address:
\centerline{\scshape Ho Law}
\medskip
{\footnotesize
% please put the address of the first author
 \centerline{School of Mathematics}
 \centerline{Georgia Institute of Technology, Atlanta, GA 30332, USA}
} % Do not forget to end the {\footnotesize by the sign }

\medskip

\centerline{\scshape Gary P. T. Choi}
\medskip
{\footnotesize
 \centerline{Department of Mathematics}
 \centerline{Massachusetts Institute of Technology, Cambridge, MA 02139, USA}
}
\medskip

\centerline{\scshape Ka Chun Lam}
\medskip
{\footnotesize
 \centerline{Machine Learning Team}
 \centerline{National Institute of Mental Health, Bethesda, MD 20892, USA}
}
\medskip

\centerline{\scshape Lok Ming Lui}
\medskip
{\footnotesize
 \centerline{Department of Mathematics}
 \centerline{The Chinese University of Hong Kong, Shatin, NT, Hong Kong, China}
}

\bigskip

\begin{abstract}
Image registration has been widely studied over the past several decades, with numerous applications in science, engineering and medicine. Most of the conventional mathematical models for large deformation image registration rely on prescribed landmarks, which usually require tedious manual labeling and are prone to error. In recent years, there has been a surge of interest in the use of machine learning for image registration. In this paper, we develop a novel method for large deformation image registration by a fusion of quasiconformal theory and convolutional neural network (CNN). More specifically, we propose a quasiconformal energy model with a novel fidelity term that incorporates the features extracted using a pre-trained CNN, thereby allowing us to obtain meaningful registration results without any guidance of prescribed landmarks. Moreover, unlike many prior image registration methods, the bijectivity of our method is guaranteed by quasiconformal theory. Experimental results are presented to demonstrate the effectiveness of the proposed method. More broadly, our work sheds light on how rigorous mathematical theories and practical machine learning approaches can be integrated for developing computational methods with improved performance. 
\end{abstract}

%%%%%%%%%%%%%%%%%%%%%%%%%%%%%%%%%%%%%%%%%%%%%%%%%%%%%%%%%
\section{Introduction} \label{introduction}
Image registration aims at establishing a meaningful correspondence between two images based on a given metric. Since the 1990s, it has been extremely useful in medical imaging, computer graphics and computer vision. For instance, one can register two medical images of a patient at different time points for disease diagnosis. One can also utilize image registration for creating animations and for tracking objects in different video frames. In particular, it is usually desirable but also more challenging to consider \emph{non-rigid} image registrations, which involve transformations beyond simple translation, rotation and scaling, such as affine maps, diffeomorphic maps, conformal maps and quasiconformal maps.

Conventional mathematical models for non-rigid image registration can be categorized into three major types: landmark-based methods, intensity-based methods, and hybrid methods. Landmark-based methods use prescribed feature points (also known as \emph{landmarks}) to guide the registration, while intensity-based methods use only the intensity of the images to compute the registration. Hybrid methods take the advantages of the two methods by considering both the image intensity and the landmark correspondences, which make them particularly effective for the case where the source image and the target image are assumed to be with a large deformation. However, in practice, manual landmark labeling is time-consuming and easily affected by human errors. 

In this work, we propose a novel variational model for large deformation image registration by combining quasiconformal theory and convolutional neural network (CNN). Specifically, instead of using a landmark-based fidelity term, we propose a novel fidelity term that uses feature vectors obtained from a truncated classification network to guide the registration. As the CNN features do not necessary yield a smooth and bjiective map, quasiconformal energy terms are included in the variational model to reduce quasiconformal distortion and ensure bijectivity. Our experiments show that even only using a common pre-trained classification network, the proposed variational model already works very well for registering a large variety of images. 

The rest of the paper is organized as follows. In Section~\ref{sect:contribution}, we highlight the contributions of our work. In Section~\ref{sect:prevwork}, we review the related works on image registration and machine learning. In Section~\ref{sect:background}, we introduce the theory of quasiconformal maps and convolutional neural networks. The proposed model and the implementation are then described in Section~\ref{sect:main} and Section~\ref{sect:implementation} respectively. In Section~\ref{sect:experiment}, we demonstrate the effectiveness of the proposed algorithm using various experiments. We conclude the paper and discuss possible future works in Section~\ref{sect:conclusion}.

%%%%%%%%%%%%%%%%%%%%%%%%%%%%%%%%%%%%%%%%%%%%%%%%%%%%%%%%%%%
\section{Contributions} \label{sect:contribution}
The contributions of our work are as follows:
\begin{enumerate}[(i)] 
    \item Our proposed variational model involves a novel combination of quasiconformal maps and machine learning for large deformation image registration. Specifically, we propose a fidelity term to incorporate the features extracted using a pre-trained classification CNN into our quasiconformal energy model.
    \item Unlike many prior image registration models, our method does not rely on any prescribed landmarks. 
    \item Even only using a common pre-trained classification network for getting the CNN feature vectors, our proposed model is capable of yielding accurate registration results.
    \item The bijectivity of the proposed method is theoretically guaranteed by quasiconformal theory.
    \item Our work demonstrates how machine learning approaches and mathematical theories can be integrated for the development of useful computational methods.
\end{enumerate}
%%%%%%%%%%%%%%%%%%%%%%%%%%%%%%%%%%%%%%%%%%%%%%%%%%%%%%%%%%%%%
\section{Related works}\label{sect:prevwork}
\subsection{Non-rigid image registration}
Over the past several decades, image registration has been extensively studied (see~\cite{brown1992survey,maintz1998survey,zitova2003image,sotiras2013deformable} for detailed surveys). In~\cite{opticalflow}, Horn and Schunck presented a method for registering images using optical flow. In~\cite{johnson2002consistent}, Johnson and Christensen proposed a method for consistent landmark and intensity-based image registration using thin-plate spline (TPS)~\cite{bookstein1989principal}. In~\cite{lddmm}, Beg \emph{et~al.} proposed the large deformation diffeomorphic metric mapping (LDDMM) algorithm for image warping. Joshi and Miller~\cite{joshi2000landmark} used large deformation diffeomorphisms for matching landmarks. Other publicly available image registration tools include Elastix~\cite{klein2009elastix}, deformable image registration using discrete optimization (DROP)~\cite{glocker2011deformable}, flexible algorithms for image registration (FAIR)~\cite{modersitzki2009fair} etc. 

Diffeomorphic Demons (DDemons), developed by Vercauteren \emph{et~al.}~\cite{demons}, is a well-known non-parametric diffeomorphic image registration method stemming from the work of Thirion~\cite{thiriondemon}. By combining DDemons and quasiconformal theory, Lam and Lui~\cite{QCHR} proposed the quasiconformal hybrid registration (QCHR) method that reduces the local geometric distortion of the registration map. The QCHR method has been successfully applied to different registration and shape analysis tasks~\cite{yung2018efficient,choi2020shape}. However, in case manual landmark labeling is not available and the main features in the input images do not overlap, these methods may not work well. More specifically, if there is little or no overlap of the features in the source image and the target image, these methods will only shrink the features and yield a meaningless registration map. Therefore, these methods can only handle large deformation image registration with the presence of prescribed landmarks.

\subsection{Imaging and machine learning}
In recent years, the use of neural networks has become increasingly popular for imaging and computer vision~\cite{ptdescriptor,pnnet,matchnet,localdescriptor,comparepatches,jia2021regularized}. In particular, there have been a number of works that use deep learning framework for image and surface registration, such as DLIR~\cite{bob},  VoxelMorph~\cite{voxelmorph}, FAIM~\cite{faim}, Quicksilver~\cite{yang2017quicksilver} and cycle-consistent training~\cite{kuang2019cycle}. These methods usually require carefully designing the neural network architecture and performing extensive training for achieving the registration. By contrast, the goal of our work is to develop an energy model for image registration without having to build any registration network or perform any training. Our method only uses feature vectors extracted by a pre-trained classification network and solves an energy minimization problem to obtain the registration.

\section{Theoretical background}\label{sect:background}
\subsection{Quasiconformal theory}
In this work, we use quasiconformal maps to obtain diffeomorphic image registrations with large deformations. In this section, we describe some related concepts in quasiconformal theory. Readers are referred to~\cite{Gardiner,Lehto} for more details.

\begin{figure}[t]
\centering
\includegraphics[width=0.9\textwidth]{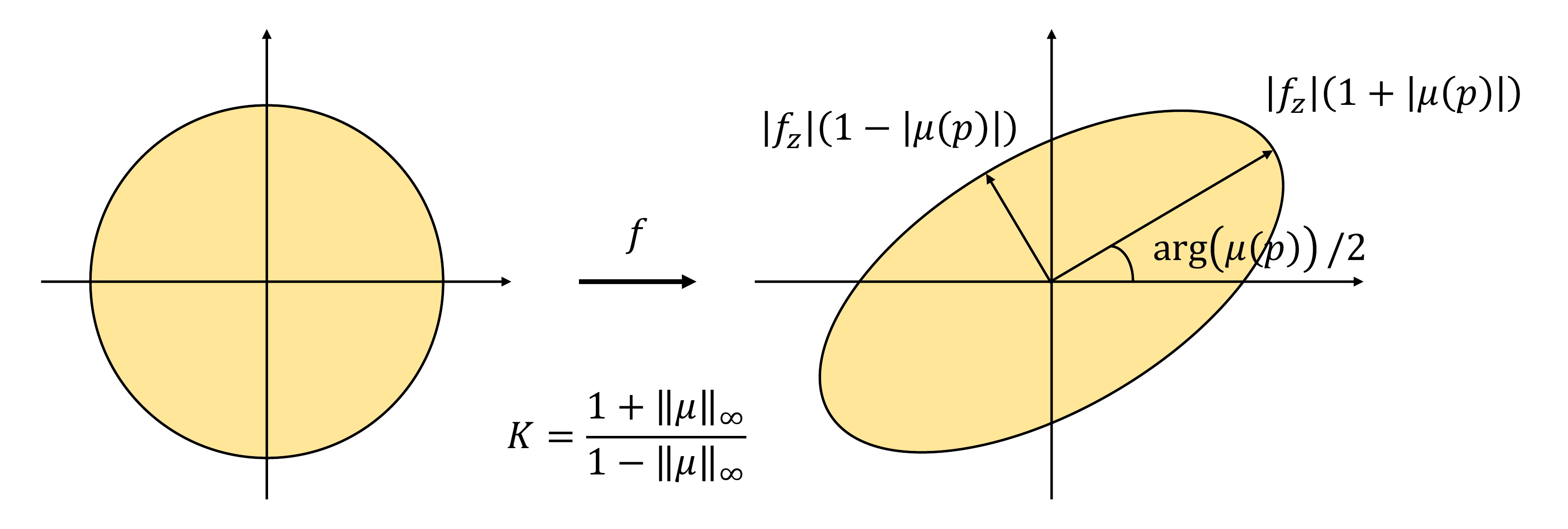}
\caption{An illustration of how the Beltrami coefficient $\mu$ determines the conformality distortion. Under a quasiconformal map $f$, an infinitesimal circle around a point $p$ is mapped to an infinitesimal ellipse centered at $f(p)$, where the major axis length and the minor axis length are given by $|f_z(p)|(1+|\mu(p)|)$ and $|f_z(p)|(1-|\mu(p)|)$. Therefore, the maximal dilation of $f$ is $K(f) = \frac{1+||\mu||_{\infty}}{1-||\mu||_{\infty}}$. Also, the orientation change of the major axis of the ellipse is given by $\arg(\mu(p))/2$.}
\label{fig:illustration1}
\end{figure}

Let $f:\mathbb{C} \to \mathbb{C}$ be a holomorphic function on the complex plane with $f(z) = f(x,y) = u(x,y) + i v(x,y)$, where $z = x + iy$, $u,v$ are real-valued functions, and its derivative is non-zero everywhere. $f$ is said to be \emph{conformal} if it satisfies the Cauchy--Riemann equations
\begin{equation}\label{eqt:cauchyriemann}
\frac{\partial u}{\partial x} = \frac{\partial v}{\partial y} \ \ \text{ and } \ \ 
    \frac{\partial u}{\partial y} = -\frac{\partial v}{\partial x}. 
\end{equation}
If we denote 
\begin{equation}
    \frac{\partial f}{\partial \overline{z}} = \frac{1}{2}\left(\frac{\partial f}{\partial x} + i\frac{\partial f}{\partial y}\right) \ \ \text{ and } \ \ \frac{\partial f}{\partial z} = \frac{1}{2}\left(\frac{\partial f}{\partial x} - i\frac{\partial f}{\partial y}\right),
\end{equation}
then Equation~\eqref{eqt:cauchyriemann} can be rewritten as
\begin{equation}
    \frac{\partial f}{\partial \overline{z}} = 0.
\end{equation}

Quasiconformal maps are a generalization of conformal maps. Intuitively, the first order approximations of conformal maps take small circles to small circles, while the first order approximations of quasiconformal maps take small circles to small ellipses of bounded eccentricity~\cite{Gardiner}. Mathematically, an orientation-preserving homeomorphism $f:\mathbb{C} \to \mathbb{C}$ is said to be \emph{quasiconformal} if it satisfies the Beltrami equation:
\begin{equation}\label{beltramieqt}
\frac{\partial f}{\partial \overline{z}} = \mu(z) \frac{\partial f}{\partial z}
\end{equation}
for some complex-valued function $\mu$ (called the \emph{Beltrami coefficient}) with $||\mu||_{\infty}< 1$. Note that $\mu$ measures how far the map at each point is deviated from a conformal map. In particular, the map $f$ is conformal around a small neighborhood of $p$ if and only if $\mu(p) = 0$. One may express $f$ around a point $p$ with respect to its local parameter as follows (see Fig.~\ref{fig:illustration1}):
\begin{equation}
\begin{split}
f(z) & = f(p) + f_{z}(p)z + f_{\overline{z}}(p)\overline{z} = f(p) + f_{z}(p)(z + \mu(p)\overline{z}).
\end{split}
\end{equation}
In other words, locally $f$ can be considered as a map composed of a translation by $f(p)$ together with a stretch map $S(z)=z + \mu(p)\overline{z}$ with a multiplication of $f_z(p)$. Because of the factor $\mu(p)$ in $S(z)$, $f$ maps a small circle to a small ellipse. More specifically, the maximal magnification factor is $|f_z(p)|(1+|\mu(p)|)$ and the maximal shrinkage factor is $|f_z(p)|(1-|\mu(p)|)$. The maximal dilation of $f$ is then given by
\begin{equation}
K(f) = \frac{1+||\mu||_{\infty}}{1-||\mu||_{\infty}}.
\end{equation}
Also, the orientation change of the major axis of the ellipse is given by $\arg(\mu(p))/2$. This shows that the Beltrami coefficient $\mu$ provides us with useful information of the quasiconformality of the mapping $f$.

Besides, given a Beltrami coefficient $\mu:\mathbb{C}\to \mathbb{C}$ with $\|\mu\|_\infty < 1$, there is always a quasiconformal mapping from $\mathbb{C}$ onto itself which satisfies the Beltrami equation~\eqref{beltramieqt} in the distribution sense~\cite{Gardiner}. More precisely, we have the following theorem:

\begin{theorem}[Measurable Riemann Mapping Theorem] \label{thm:Beltrami}
Suppose $\mu: \overline{\mathbb{C}} \to \overline{\mathbb{C}}$ is Lebesgue measurable and satisfies $\|\mu\|_\infty < 1$. Then there exists a quasiconformal homeomorphism $\phi$ from $\overline{\mathbb{C}}$ onto itself, which is in the Sobolev space $W^{1,2}(\overline{\mathbb{C}})$ and satisfies the Beltrami equation~\eqref{beltramieqt} in the distribution sense. Furthermore, by fixing 0, 1 and $\infty$, the quasiconformal homeomorphism $\phi$ is uniquely determined for any given $\mu$.
\end{theorem}
Theorem~\ref{thm:Beltrami} suggests that under suitable normalization, a homeomorphism from $\mathbb{C}$ onto itself can be uniquely determined by its associated Beltrami coefficient. 

\subsection{Linear Beltrami solver}
The Linear Beltrami solver (LBS)~\cite{QCHR} provides us with an efficient way for reconstructing a quasiconformal map $f(z) = u(x,y) + i v(x,y)$ given a Beltrami coefficient $\mu_f$. First, note that the Beltrami equation~\eqref{beltramieqt} can be rewritten as
\begin{equation}
\mu_f = \frac{(u_x - v_y) + i(v_x + u_y)}{(u_x + v_y) + i(v_x - u_y)}.
\end{equation}
Now, let $\mu_f(z) = \rho(z) + i\tau(z)$ where $\rho, \tau$ are real-valued functions. We can then express $v_x$ and $v_y$ as linear combinations of $u_x$ and $u_y$:
\begin{equation}\label{eqt:linearB1cont}
\begin{split}
-v_y & = \alpha_1 u_x + \alpha_2 u_y;\\
v_x & = \alpha_2 u_x + \alpha_3 u_y,
\end{split}
\end{equation}
where $\alpha_1 = \frac{(\rho -1)^2 + \tau^2}{1-\rho^2 - \tau^2} $, $\alpha_2 = -\frac{2\tau}{1-\rho^2 - \tau^2} $, and $\alpha_3 = \frac{1+2\rho+\rho^2 +\tau^2}{1-\rho^2 - \tau^2} $. Similarly, we can express $u_x$ and $u_y$ as linear combinations of $v_x$ and $v_y$:
\begin{equation} \label{eqt:linearB2cont}
\begin{split}
u_y & = \alpha_1 v_x + \alpha_2 v_y;\\
-u_x & = \alpha_2 v_x + \alpha_3 v_y.
\end{split}
\end{equation}
Since $\nabla \cdot \left(\begin{array}{c}
-v_y\\
v_x \end{array}\right) = 0$ and $\nabla \cdot \left(\begin{array}{c}
u_y\\
-u_x \end{array}\right) = 0$, we have
\begin{equation}\label{eqt:BeltramiPDE}
\nabla \cdot \left(A \left(\begin{array}{c}
u_x\\
u_y \end{array}\right) \right) = 0\ \ \mathrm{and}\ \ \nabla \cdot \left(A \left(\begin{array}{c}
v_x\\
v_y \end{array}\right) \right) = 0,
\end{equation}
where $A = \left( \begin{array}{cc}\alpha_1 & \alpha_2\\
\alpha_2 & \alpha_3 \end{array}\right)$.

In the discrete case, the elliptic PDEs~\eqref{eqt:BeltramiPDE} can be discretized as sparse positive definite linear systems and hence can be efficiently solved. Readers are referred to~\cite{QCHR,choi2015flash} for more details.

\subsection{Convolutional neural networks (CNNs)}
Another important component in our proposed method is the use of features extracted by CNNs. A comprehensive introduction to CNN can be found in~\cite{Goodfellow-et-al-2016}. Here, we introduce the concept of \emph{receptive field} in CNN, which is particularly related to our work.

In the context of neural network for imaging, a receptive field is an area that is being read by the network at a specific layer. More specifically, the network takes in an image pixel by pixel initially and treats each pixel as a 3-dimensional vector for RGB images. Then after certain layers of convolution and pooling, the network starts to read the image region by region and it does so by assigning each region a high dimensional vector. This region is called a receptive field. In general, at the end of a CNN, the size of a receptive field would usually outgrow the size of input image, and different receptive fields overlap each other.

An illustration of receptive field is shown in Fig.~\ref{fig:receptive field}. Suppose the input image is of size $5 \times 5$, kernel of size $3 \times 3$, padding of size $1$ and stride of size $2 \times 2$. After two convolutions, the purple pixels, including the padding, on the rightmost grid contribute to the value of the top left corner of the $2 \times 2$ grid. Thus, these purple pixels in the original input image form a receptive field of that top left value. 

\begin{figure}[t]
\centering
\includegraphics[width=\textwidth]{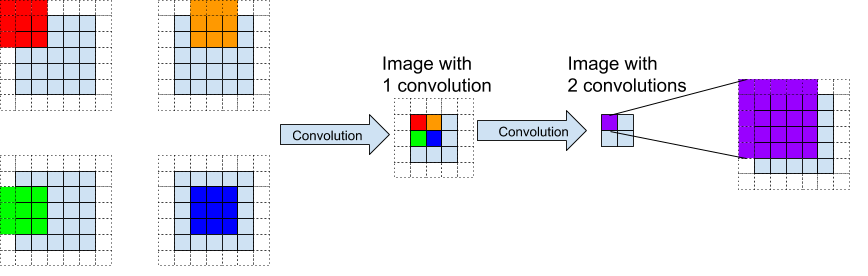}
\caption{An illustration of receptive field.}
\label{fig:receptive field}
\end{figure}

%%%%%%%%%%%%%%%%%%%%%%%%%%%%%%%%%%%%%%%%%%%%%%%%%%%%%%%%%%%%%%%%%%%
\section{Proposed method}\label{sect:main}
In this section, we describe our proposed quasiconformal energy model with CNN features for large deformation image registration.
  
\subsection{Measuring the correlation of image patches using a pre-trained network and receptive fields}
In~\cite{Rocco17}, Rocco \emph{et~al.} proposed a method for obtaining similarity information from two images using some typical image classification networks. The main idea is to extract a feature vector from each receptive field of an image using a truncated classification network, and then compute the similarities between each receptive field of the source image and the target image using inner products. Since receptive fields are large and usually overlap with each other, this gives a way of roughly evaluating the interaction between patches on a single image. Motivated by this approach, in this work, we propose to partition the images and then evaluate the correlation of the image patches with the aid of a classification network. We will then incorporate the similarity information into a quasiconformal energy model for computing the registration. Unlike the work by Rocco \emph{et~al.}, our approach requires a higher level of preciseness as we will be defining a bijective correspondence between two images. Thus, instead of allowing the receptive field to grow with little control, we partition the images into patches to feed into network, so as to gain back the control of how large and where the network reads for generating a feature vector. 

If we feed an image into a typical classification neural network such as ResNet~\cite{he2016deep}, VGG~\cite{simonyan2014very} or DenseNet~\cite{huang2017densely}, then up to some layer in the middle, we can obtain a high dimension vector for each receptive field and stack them vertically to yield a high dimension vector. Fig.~\ref{fig:feature vectors} illustrates the process of producing such a high dimension vector: feeding an $h \times w $ image through a truncated classification network generates a 3D array of size $m \times n \times d$, where $m,n$ are respectively the number of receptive fields along the width and height of the input image depending on stride, kernel and padding size, and $d$ is the dimension of the feature vector depending on the architecture of the network. Here, we remark that Fig.~\ref{fig:feature vectors} does not reflect the actual architecture of a classification network. The red column in the cuboid represents a feature vector of the red receptive field in the input image, and for each vector we stack them up vertically. Mathematically, this stacking can be viewed as an isomorphism that maps the output array to a vector in $\mathbb{R}^{mnd}$. The inner product of this supreme feature vector of two images is then the sum of the correlation scores of the corresponding pairs of receptive fields. This provides us with a quantitative way for measuring the correlation of two images, which plays an important role in our proposed energy model described in the next section.

\begin{figure}[t]
\centering
\includegraphics[width=0.9\textwidth]{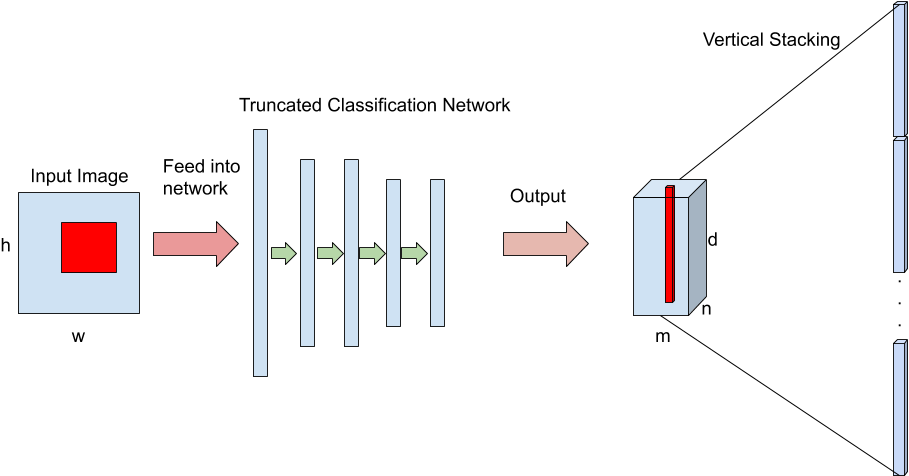}
\caption{An illustration of stacking feature vectors. An image with height $h$ and width $w$ is fed into a truncated classification network, which generates a 3D array of size $m \times n \times d$, where $m,n$ are respectively the number of receptive fields along the width and height of the input image depending on stride, kernel and padding size, and $d$ is the dimension of the feature vector depending on the architecture of the network. the 3D array is then stacked up to form the output feature vector in $\mathbb{R}^{mnd}$.}
\label{fig:feature vectors}
\end{figure}

\subsection{Quasiconformal energy model with CNN features}\label{model}
The classification network gives us a reliable quantitative measurement of the correlation between two images. The next crucial step is to incorporate this useful information into the registration model. Let $\mathcal{M}$ and $\mathcal{S}$ be the moving and static images respectively. To find a diffeomorphism $f:\mathcal{M} \to \mathcal{S}$ such that $\|\mu(f) \|_{\infty} < 1$, we propose the following model:
\begin{equation} \label{raw model}
    f = \argmin_{g:\mathcal{M} \to \mathcal{S}} E_{C}(\mu(g)),
\end{equation}
where
\begin{equation} \label{energy}
    E_{C}(\mu(g)) = \int_{\mathcal{M}} \left(|\nabla \mu |^{2} + \alpha | \mu |^{2} + \beta (I_{\mathcal{M}} - I_{\mathcal{S}}(g))^{2}\right) \mathrm{d}z + \gamma  \|C \otimes D(g) - C \|^{2}_{F}.
\end{equation}
Here, $\alpha,\beta,\gamma$ are some positive real numbers, $I_{\mathcal{M}}$ and $I_{\mathcal{S}}$ denote the intensity of the images, $C$ is an $m\times m$ matrix where $m$ is the number of image patches for each image, $D: \{ g:\mathcal{M} \to \mathcal{S} \} \to [0,1]^{m \times m}$ maps $g$ to an $m \times m$ matrix, and $\otimes$ is the Hadamard product with
\begin{equation}
    (A \otimes B)_{ij} = (A)_{ij} (B)_{ij}
\end{equation}
for any matrices $A, B$ of the same dimension. The key idea of the proposed model is to use the features extracted by a robust CNN to guide the registration process via the last term in the energy~\eqref{energy} while maintaining the diffeomorphic property of the deformation using the first two terms of the integral based on quasiconformal theory.

More specifically, the first term $|\nabla \mu|^2$ of the integral in the energy~\eqref{energy} is used for enhancing the smoothness of $f$. Recall that by Theorem~\ref{thm:Beltrami}, there exists a correspondence between quasiconformal mappings and the Beltrami coefficients, which measure the local geometric distortion of the mappings. Therefore, a smaller $|\nabla \mu|^2$ gives a smoother quasiconformal map $f$. The second term $|\mu|^2$ of the integral is used for reducing the conformality distortion of the registration. As described in Section~\ref{sect:background}, a map is conformal if and only if $\mu = 0$. A smaller $|\mu|^2$ gives a map with a lower local geometric distortion, which is desirable. The third term $(I_{\mathcal{M}} - I_{\mathcal{S}}(g))^{2}$ of the integral aims to reduce the intensity difference of the images. 

The last term $\|C \otimes D(g) - C \|^{2}_{F}$ in the proposed energy~\eqref{energy} is a novel fidelity term that gives a descent direction to guide the registration based on the correlation of different regions of two images, which plays an important role in our proposed method. More specifically, the feature vectors encode important hidden information extracted by a classification CNN from big image datasets, which may not be easily represented using conventional mathematical quantities such as curvature or gradient. With this novel fidelity term, we can utilize such information and combine it with quasiconformal maps for achieving a meaningful registration. One key feature of our proposed model is that it allows us to save the effort of going through the training of a neural network for registration, as there are already many publicly available classification networks well-trained on giving meaningful feature vectors that we can re-use. Another advantage of this approach is the generality: for most of the common images, the CNN should give fairly good guidance for rudimentary correspondence on two images even if it is not trained on a specific type of images. We remark that in case we want to focus on a more specific type of images, it is also possible to use a pre-trained classification network on such type of images to get the feature vectors. Below, we first introduce the definition of $C$ and $D(g)$, and then explain the use of the Hadamard product.

Roughly speaking, the matrix $C$ is used for deciding how trustworthy the correlation between a pair of patches is when compared with the others. To define $C$, we first denote $H: P \to \mathbb{R}^{d}$ as a truncated classification network that sends an image patch $P$ to a $d$-dimensional feature vector. Define a matrix $\tilde{C} \in \mathbb{R}^{m \times m}$, where $m$ is the number of patches we partition on each image~\cite{Rocco17}: 
\begin{equation}\label{eq: hat c}
\left( \tilde{C} \right)_{ij} = \left\langle \frac{H(P^{\mathcal{M}}_{i})}{|H(P^{\mathcal{M}}_{i})|} , \frac{H(P^{\mathcal{S}}_{j})}{|H(P^{\mathcal{S}}_{j})|} \right\rangle,
\end{equation}
where $P_i^{\mathcal{M}}$ is the $i$-th image patch of the moving image $\mathcal{M}$ and $P_j^{\mathcal{S}}$ is the $j$-th image patch of the static image $\mathcal{S}$. It is noteworthy that passing an image patch into a truncated network would yield a 3D matrix of size $\mathbb{R}^{h \times w \times m}$, where $h$ and $w$ are the height and width of each image patch, which can be viewed as $\mathbb{R}^{d}$ where $d = hwm$. This allows us to define the inner product of 3D matrices as the usual inner product of vectors.

Then, we normalize each row of $\tilde{C}$ by
\begin{equation}
    \left( \hat{C} \right)_{ij} = \frac{\tilde{C}_{ij} - \mu_{i}}{\sigma_{i}},
\end{equation}
where $\sigma_{i}$ and $\mu_{i}$ are the standard deviation and mean of row $i$ respectively. The reason for normalising each row of $\tilde{C}$ is to spread out the similarity of patches. For example, if a row in $\tilde{C}$ happens to have most entries fluctuate within a small range, it is not easy to assess the similarity between the patches. Normalising the matrix row-by-row is beneficial for finding descent direction.

In many applications, it is common that there is a large background region with a uniform color in the images to be registered. For instance, the X-ray, CT and MRI scans in medical imaging are usually with a completely black background. For such uniform background patches in the moving image, it is expected that their corresponding rows in $\hat{C}$ would be exactly the same and hence are not useful for the registration process. Since it is expected that the feature vectors extracted from different background patches on the source image would give an inner product almost the same with the feature vectors extracted from any patches on the target image, there is a need to further prune the meaningless rows in $\hat{C}$. To neglect these background rows, we define an $m \times m$ elimination matrix $E$ with
\begin{equation}\label{eq: elimination matrix}
    (E)_{ij} = \begin{cases}
     0, \text{ if $i \neq j$}, \\
     0, \text{ if $i = j$ and row $i$ is not unique in $\hat{C}$},\\
     1, \text{ if $i = j$ and row $i$ is unique in $\hat{C}$}. \\
     \end{cases}
\end{equation}
Now, we can remove the background rows in $\hat{C}$ and obtain an updated matrix $C^{\dagger} \in \mathbb{R}^{m \times m}$ by
\begin{equation}
    C^{\dagger} = E \hat{C}.
\end{equation}
Finally, the matrix $C$ in the proposed energy~\eqref{energy} is constructed by sparsifying $C^{\dagger}$: For each row in $C^{\dagger}$, we keep the largest entry and set all other entries to be zero. The reason for sparsifying $C^{\dagger}$ is essentially for computation efficiency. More specifically, after the sparsification, we can obtain a descent direction based on the most significant correlation of the patches for the energy minimization in the later steps. Without the sparsification step, the other non-zero entries would generate some misleading direction in finding the descent direction and hence affect the efficiency.

To define $D(g)$ in the energy~\eqref{energy}, we denote the center of image patch $P_i^{\mathcal{M}}$ and $P_j^{\mathcal{S}}$ on the moving and static images as $x_i^{\mathcal{M}}$ and $x_j^{\mathcal{S}}$ respectively. The mapping $D: \{ g:\mathcal{M} \to \mathcal{S} \} \to [0,1]^{m \times m}$ is defined by
\begin{equation}\label{eq: D def}
    \left(D(g) \right) _{ij} = \exp \left( - \frac{\| g(x_i^{\mathcal{M}}) - x_j^{\mathcal{S}} \|^{2}_{2}}{\sigma^2} \right),
\end{equation}
where $\sigma$ is a small number to be chosen.

As for the purpose of the Hadamard product in the proposed fidelity term $\|C \otimes D(g) - C \|^{2}_{F}$, note that $\left(D(g) \right) _{ij} = 1$ if and only if $g(x_i^{\mathcal{M}}) = x_j^{\mathcal{S}}$, i.e. $g$ maps the patch $P_i^{\mathcal{M}}$ on the moving image $\mathcal{M}$ to the patch $P_j^{\mathcal{S}}$ on static image $\mathcal{S}$ exactly. Now, note that the fidelity term can be expressed as a weighted sum of correspondences:
\begin{equation}
    \|C \otimes D(g) - C \|^{2}_{F} = \sum_{ij} C^{2}_{ij} (\left(D(g) \right) _{ij}-1)^2.
\end{equation}
Therefore, with the use of the Hadamard product, $(C)_{ij}$ can serve as a weighting factor to determine how exact the mapping between $x_i^{\mathcal{M}}$ and $x_j^{\mathcal{S}}$ should be. More specifically, if $(C)_{ij}$ is large, i.e. the correlation between the two patches is high as determined by the pre-trained network, then $\left(D(g) \right) _{ij}$ should be as close to 1 as possible; if $(C)_{ij}$ is small, the requirement for $\left(D(g) \right)_{ij} = 1$ can be relaxed.

Moreover, in the optimization process, note that the proposed fidelity term can be handled using gradient descent. The descent direction (denoted as $\mathrm{d}f_{\mathcal{W}}: \{x_i^{\mathcal{M}} \}_{i=1}^{m} \to \mathbb{R}^{2}$) can be explicitly written as
\begin{equation}\label{c term descent}
   \mathrm{d}f_{\mathcal{W}}(x_i^{\mathcal{M}}) = \frac{4}{\sigma^{2}}\sum_{j} (C)^{2}_{ij}\left( \exp \left( - \frac{\| g(x_i^{\mathcal{M}}) - x_j^{\mathcal{S}} \|^{2}_{2}}{\sigma^2} \right) - 1 \right) (g(x_i^{\mathcal{M}}) - x_j^{\mathcal{S}}).
\end{equation}
Therefore, if $(C)_{ij} \approx 0$, i.e. the network determines that the patch $P_i^{\mathcal{M}}$ on the moving image is not correlated to the patch $P_j^{\mathcal{S}}$ on the target image, the corresponding descent direction for the $(i,j)$ pair will be close to 0. In other words, the Hadamard product allows us to eliminate unwanted descent directions based on the information provided by the CNN encoded in the matrix $C$.

On the contrary, if we replace the proposed fidelity term with $\|D(g) - C\|^{2}_{F}$, then under the same condition of $(C)_{ij} \approx 0$, the descent direction will become
\begin{equation}
\begin{aligned}
    \mathrm{d}f(x_i^{\mathcal{M}}) &= \frac{4}{\sigma^{2}}\sum_{j} \left( \exp \left( - \frac{\| g(x_i^{\mathcal{M}}) - x_j^{\mathcal{S}} \|^{2}_{2}}{\sigma^2} \right) - (C)_{ij} \right) (g(x_i^{\mathcal{M}}) - x_j^{\mathcal{S}}) \\
    &\approx \frac{4}{\sigma^{2}}\sum_{j} \exp \left( - \frac{\| g(x_i^{\mathcal{M}}) - x_j^{\mathcal{S}} \|^{2}_{2}}{\sigma^2} \right) (g(x_i^{\mathcal{M}}) - x_j^{\mathcal{S}}) \\
    &\neq 0 \qquad \text{if $g(x_i^{\mathcal{M}}) \neq x_j^{\mathcal{S}}$}. 
\end{aligned}
\end{equation}
This shows that the descent direction is noisy in case the Hadamard product is not used.

After justifying our proposed energy model~\eqref{energy}, it is natural to ask whether the existence of the minimizer~\eqref{raw model} is guaranteed. We have the following result:
\begin{proposition}[Existence of minimizer]
Let
\begin{equation}\label{admissible set}
   \mathcal{A}  = \{ \nu \in C^{1}(\Omega_{1}) \ : \ \|D\nu\|_{\infty} \leq C_{1}; \| \nu \|_{\infty} \leq 1 - \epsilon \}
\end{equation}
for some $C_{1} > 0$ and some small $\epsilon > 0$. Then the proposed energy $E_C$ has a minimizer in $\mathcal{A} \subset C^{1}(\Omega_{1})$ if $I_{\mathcal{M}},I_{\mathcal{S}}$ are $L^{2}$ functions.
\end{proposition}
\begin{proof}

First, note that if $I_{\mathcal{M}},I_{\mathcal{S}}$ are $L^{2}$ functions, then $E_C$ is continuous. Therefore, to prove that $E_C$ has a minimizer in $\mathcal{A}$, it suffices to prove that $\mathcal{A}$ is compact. The proof follows from the proof of Proposition 4.1 in~\cite{QCHR} and is outlined below.

Note that the Beltrami coefficient $\mu^*$ associated with the unique Teichm\"uller map is in $\mathcal{A}$ by choosing $C_1 = \|D\mu^*\|_{\infty}$ and $\epsilon \in (0, 1-\|\mu^*\|_{\infty})$. Therefore, $\mathcal{A} \neq \emptyset$.

Now, we prove that $\mathcal{A}$ is complete. Let $\{\nu_n\}_{n=1}^{\infty}$ be a Cauchy sequence in $\mathcal{A}$ under the norm $\|\cdot\|_s := \|\cdot\|_{\infty} + \|D \cdot \|_{\infty}$. Then $\{D\nu_n\}_{n=1}^{\infty}$ is also a Cauchy sequence under the norm $\|\cdot\|_s$. As $\mathcal{A} \subset W^{1,\infty}(\Omega_1)$ and $W^{1,\infty}$ is complete, it follows that $\{D\nu_n\}_{n=1}^{\infty}$ converges to some $g \in W^{1,\infty}$ uniformly. As $D\nu_n$ is continuous, $g$ is also continuous. Also, $\{\nu_n\}_{n=1}^{\infty}$ is convergent. Hence, we have $\nu_n \to \nu$ and $D\nu_n \to D\nu$ uniformly for some $nu \in C^1(\Omega_1)$. Moreover, as $\|v_n\|_{\infty} \leq 1-\epsilon$ for all $n$, we have $\|\nu\|_{\infty} \leq 1-\epsilon$. Similarly, as $\|D\nu_n\|_{\infty} \leq 1-\epsilon$ for all $n$, we have $\|D\nu\|_{\infty} \leq C_1$. Therefore, $\nu \in \mathcal{A}$ and hence $\mathcal{A}$ is Cauchy complete. We can also show that $\mathcal{A}$ is totally bounded using the argument in the proof of Proposition 4.1 in~\cite{QCHR}. Hence, $\mathcal{A}$ is compact.

As $E_C$ is continuous in $\mathcal{A}$ and $\mathcal{A}$ is compact, we conclude that $E_C$ has a minimizer in $\mathcal{A}$.
\end{proof}

As a remark, a major difference between the proposed energy model and the landmark- and intensity-based quasiconformal hybrid registration (QCHR) method \cite{QCHR} is that QCHR uses the prescribed feature pairs as hard landmark constraints for the energy minimization, while in our method we do not impose any landmark constraints. Using a common classification network pre-trained on very wide-ranging datasets, we obtain multiple correlated pairs of image patches. Instead of enforcing them as landmark constraints, we encode the correlation information in the proposed fidelity term in the energy~\eqref{energy}, which helps supply a good descent direction for yielding a good registration. This does not only avoid any potential non-bijectivity due to incorrectly correlated pairs but also allow the energy model to utilize the correlation information for the registration of more general images that it has not seen before. 

\subsection{Energy minimization using the penalty splitting method}\label{minimisation method}
The penalty splitting method is used for solving the energy minimization problem~\eqref{raw model}. More specifically, instead of directly minimizing the energy~\eqref{energy}, we minimize the following energy
\begin{equation} \label{splitted energy}
     E_{C}(\nu,f) = \int_{\mathcal{M}} (| \nabla \nu |^{2} + \alpha | \nu |^{2} + \rho | \nu - \mu(f) |^{2}  + \beta \left( I_{\mathcal{M}} - I_{\mathcal{S}}(f) \right)^{2}) \mathrm{d}z + \gamma \| C \otimes D(f) - C \|^{2}_{F}.
\end{equation}
Here, the new term $\rho | \nu - \mu(f) |^{2}$ forces $\mu(f)$ to closely resemble $\nu$. This allows us to consider the Beltrami coefficient $\mu$ and the mapping $f$ separately, so that the optimization can be performed more easily.

\paragraph{Minimizing over $\nu$}
If we fix $f=f_{n}$, it suffices to minimize
\begin{equation}
        E_{C}(\nu, f_{n}) = \int_{\mathcal{M}} (| \nabla \nu |^{2} + \alpha | \nu |^{2} + \rho | \nu - \mu(f_{n})|^{2}) \mathrm{d}z.
\end{equation}
As discussed in~\cite{QCHR}, the minimizer of the above energy can be obtained by solving the Euler--Lagrange equation:
\begin{equation}\label{euler lagrange eq}
    (- \Delta + 2\alpha I + 2 \rho I) \nu_{n+1} = 2\rho \mu(f_{n}).
\end{equation}
In the discrete case, Equation~\eqref{euler lagrange eq} can be discretized as a sparse linear system and be solved efficiently.

\paragraph{Minimizing over $f$} \label{minimize f}
If we fix $\nu = \nu_n$, it suffices to minimize
\begin{equation}
        E_{C}(\nu_n, f) = \int_{\mathcal{M}} (\rho | \nu_n - \mu(f) |^{2}  + \beta \left( I_{\mathcal{M}} - I_{\mathcal{S}}(f) \right)^{2}) \mathrm{d}z + \gamma \| C \otimes D(f) - C \|^{2}_{F}.
\end{equation}
This can be done by performing gradient descent on the Beltrami coefficient $\mu = \mu_f$ associated with $f$.

First, note that the descent direction for the intensity term $\left( I_{\mathcal{M}} - I_{\mathcal{S}}(f) \right)^{2}$ in the space of mappings (denoted as $\mathrm{d}{f_\mathcal{I}}$) is given by
\begin{equation}\label{intensity descent}
    \mathrm{d}f_{\mathcal{I}} = -2 ( I_{\mathcal{M}} - I_{\mathcal{S}}(f)) \nabla I_{\mathcal{S}}(f).
\end{equation}
Also, the descent direction $\mathrm{d}f_{\mathcal{W}}$ for the fidelity term $\| C \otimes D(f) - C \|^{2}_{F}$ in the space of mappings is given by Equation~\eqref{c term descent}. The next step is to obtain the descent directions in the space of Beltrami coefficients (denoted as $\mathrm{d}\mu_{\mathcal{I}}$ and $\mathrm{d}\mu_{\mathcal{W}}$) from $\mathrm{d}f_{\mathcal{I}}$ and $\mathrm{d}f_{\mathcal{W}}$. Note that $\mathrm{d}f_{\mathcal{I}}$ and $\mathrm{d}f_{\mathcal{W}}$ can be viewed as a perturbation:
\begin{equation}\label{perturbation}
       \frac{\partial(f + \mathrm{d}f_{i})}{\partial \Bar{z}} = (\mu + \mathrm{d}\mu_{i}) \frac{\partial(f + \mathrm{d}f_{i})}{\partial z}, \ \ i = \mathcal{I}, \mathcal{W}.
\end{equation}
Then, following the approach in~\cite{QCHR}, we can obtain $\mathrm{d}\mu_{i}$ as follows:
\begin{equation} \label{adjustment}
   \mathrm{d}\mu_{i} = \left( \left. \frac{\partial \mathrm{d}f_{i}}{\partial \Bar{z}} - \mu \frac{\partial \mathrm{d}f_{i}}{\partial z} \right) \right/ \frac{\partial (f + \mathrm{d}f_{i})}{\partial z}, \ \ i = \mathcal{I}, \mathcal{W}.
\end{equation}

For the term $|\nu_n - \mu(f) |^{2}$, the descent direction in the space of Beltrami coefficients (denoted as $\mathrm{d}\mu_{\mathcal{D}}$) is given by 
\begin{equation}\label{eqt:numu}
    \mathrm{d}\mu_{\mathcal{D}} = -2 ( \nu_n - \mu(f)).
\end{equation}

Combining the three descent directions in Equation~\eqref{adjustment} and Equation~\eqref{eqt:numu}, we can update $\mu$ at every gradient descent iteration as follows, for some $t_{1}, t_{2}, t_{3}$ chosen according to $\rho, \beta, \gamma$:
\begin{equation}\label{update rule}
    \mu_{n+1} = \mu_{n} + (t_{1} \mathrm{d}\mu_{\mathcal{I}} + t_{2} \mathrm{d}\mu_{\mathcal{W}} + t_{3}\mathrm{d}\mu_{\mathcal{D}}).
\end{equation}
After obtaining the new $\mu_{n+1}$, we can then reconstruct a quasiconformal map $f_{n+1}$ associated with $\mu_{n+1}$ using the LBS method~\cite{QCHR}.

\subsection{Additional intensity-based registration}
In general, the input source and target images will be well registered after we perform the minimization on the energy~\eqref{energy}. However, sometimes some details of the images may still not be perfectly matched. Therefore, we introduce an additional step to further improve the registration result by performing a fully intensity-based matching.

More specifically, at the $n$-th iteration, we use the DDemons method~\cite{demons} to register the current mapping result $f_n(\mathcal{M})$ and the target image $\mathcal{S}$, and denote the updated map as $f_{n} + \mathrm{d}f$. Then, we adopt the same strategy as described above to compute a descent direction $\mathrm{d}\nu$ in the space of Beltrami coefficients. We can then obtain an updated Beltrami coefficient by
\begin{equation}
    \nu_{n+1} = \nu_n + t \mathrm{d}\nu,
\end{equation}
from which we can reconstruct a quasiconformal map $f_{n+1}$ using the LBS method~\cite{QCHR}, with a proper truncation for enforcing that $\|\nu_{n+1} \|_{\infty} < 1$. We repeat the process until meeting the convergence condition.

Our proposed registration method is summarized in Algorithm~\ref{alg:cnnqc}.

\begin{algorithm}[h]
    \caption{Quasiconformal image registration with CNN features}
    \label{alg:cnnqc}
        \begin{algorithmic}[1]
            \renewcommand{\algorithmicrequire}{\textbf{Input:}}
            \renewcommand{\algorithmicensure}{\textbf{Output:}}
            \REQUIRE A moving image $\mathcal{M}$, a static image $\mathcal{S}$, a pre-trained CNN with truncation, the partition parameter $m$, the convergence threshold $\epsilon$, and the maximum number of iterations $n_{\max}$.
            \ENSURE  A quasiconformal map $f:\mathcal{M} \to \mathcal{S}$.
            \STATE Partition each image into $m$ pieces.
            \STATE Pass all $2m$ pieces into the truncated network and obtain the matrix $C$.
            \\ \textit{Energy minimization:}
            \STATE Initialize $\mu_{0} = 0$ and $\nu_{0} = 0$.
            \WHILE{$|\nu_{n+1}-\nu_{n}| > \epsilon $ and $n \leq n_{\max}$}
            \STATE Fixing $f = f_{n}$, obtain $\nu_{n+1}$ by solving the Euler--Lagrange equation~\eqref{euler lagrange eq}.
            \STATE Reconstruct a quasiconformal map with the Beltrami coefficient $\nu_{n+1}$ using the LBS method~\cite{QCHR}. 
            \STATE Fixing $\nu = \nu_{n+1}$, obtain $\mu(f_{n+1})$ by the gradient descent method in Equation~\eqref{update rule}.
            \STATE Reconstruct a quasiconformal map $f_{n+1}$ with the Beltrami coefficient $\mu_{n+1}$ using the LBS method~\cite{QCHR}.
            \STATE Update $n \gets n + 1$.
            \ENDWHILE
            \\ \textit{Additional intensity-based registration:}
            \WHILE{$|\nu_{n+1}-\nu_{n}| > \epsilon$ and $n \leq n_{\max}$} 
                \STATE Use the DDemons method~\cite{demons} to register $f_{n}(\mathcal{M})$ and $\mathcal{S}$ and denote the mapping as $f_{n} + \mathrm{d}f$.
                \STATE Obtain $\mathrm{d}\nu$ using the same procedure as in Equation~\eqref{perturbation} and Equation~\eqref{adjustment}.
                \STATE Update $\nu_{n+1} \gets \nu_{n} + t\mathrm{d}\nu$.
                \STATE Reconstruct a quasiconformal map $f_{n+1}$ from the Beltrami coefficient $\nu_{n+1}$ using the LBS method~\cite{QCHR}, with a proper truncation to enforce $\|\nu_{n+1} \|_{\infty} < 1$ if necessary.
                \STATE Update $n \gets n + 1$.
            \ENDWHILE
        \end{algorithmic}
\end{algorithm}
    
%%%%%%%%%%%%%%%%%%%%%%%%%%%%%%%%%%%%%%%%%%%%%%%%%%%%%%%%%%%%%%%%%%%%%%%
\section{Numerical implementation}\label{sect:implementation}
In the discrete case, we discretize the input images in the form of triangular meshes. Let $V^{1} = \{ v^{1}_{i} \}_{i=1}^{n}$, $V^{2} = \{ v^{2}_{i} \}_{i=1}^{n}$ be the vertex sets and $F^{1} = \{ T^{1}_{j} \}_{j=1}^{n}$, $F^{2} = \{ T^{2}_{j} \}_{j=1}^{n}$ be the face sets of the moving image $\mathcal{M}$ and the static image $\mathcal{S}$ respectively.

\subsection{Discretization of the Euler--Lagrange equation}
The Beltrami coefficient $\mu(T)$ is first discretized on each triangular face $T$ (see~\cite{QCHR,choi2015flash} for details). We can then compute the Beltrami coefficient on a vertex by taking the average value of $\mu$ on its one-ring neighboring faces:
\begin{equation}
    \mu(v_{i}) = \frac{1}{N_{i}} \sum_{T \in N_{i}} \mu(T),
\end{equation}
where $N_{i}$ is the collection of all faces incident to $v_{i}$. 

The discrete Laplacian operator $\Delta$ is given by
\begin{equation}
\Delta \left( f(v_{i}) \right) = \sum_{T \in N_{i}} \frac{\cot \alpha_{ij} + \cot \beta_{ij}}{2} \left( f(v_{j}) - f(v_{i}) \right),
\end{equation}
where $\alpha_{ij}$ and $\beta_{ij}$ are the two angles opposite to a common edge $[v_{i},v_{j}]$. 

One can then solve the Euler--Lagrange equation~\eqref{euler lagrange eq} on the vertices and finally discretize the solution $\nu$ on each triangular face by taking the average value at its three vertices:
\begin{equation}
    \nu(T) = \frac{1}{3} \sum_{v_{i} \in T} \nu(v_{i}).
\end{equation}

\subsection{Numerical techniques for intensity matching}
We use two different numerical optimization techniques for the intensity-based registration in the main energy minimization step and the additional intensity-matching step in Algorithm~\ref{alg:cnnqc}. For the main energy minimization step, we apply the modified Demon force by Wang \emph{et~al.}~\cite{wangdemon} to find the deformation:
\begin{equation}
    u = \frac{(I_{\mathcal{M}} - I_{\mathcal{S}}) \nabla I_{\mathcal{S}}}{| \nabla I_{\mathcal{S}} |^{2} + \alpha^{2}(I_{\mathcal{M}} - I_{\mathcal{S}})^{2}} + \frac{(I_{\mathcal{M}} - I_{\mathcal{S}}) \nabla I_{\mathcal{M}}}{| \nabla I_{\mathcal{M}} |^{2} + \alpha^{2}(I_{\mathcal{M}} - I_{\mathcal{S}})^{2}}.
\end{equation}
This modified Demon force is applied when finding $\mathrm{d}\mu_{\mathcal{I}}$ in the update rule (\ref{update rule}). This force is good for maintaining the diffeomorphic property of the mapping, the convergence speed as well as the stability with gradient descent. However, for the additional fully intensity-based matching step, as there can be multiple local minima for the intensity function, the gradient descent method does not necessarily yield an optimal result. Therefore, we adopt the BFGS optimization scheme~\cite{code}, which takes a longer time but achieves a more accurate intensity-based registration result.

\subsection{Gradient descent for the fidelity term}
For the proposed fidelity term, the descent direction $\mathrm{d}f(x_i^{\mathcal{M}})$ is given by Equation~\eqref{c term descent}. In practice, we notice that Equation~\eqref{c term descent} may sometimes give non-orientation preserving descent directions in some local neighborhoods so that there may be triangle fold-overs in the underlying mesh, which require additional procedures of truncating the associated Beltrami coefficients and hence affect the convergence of the computation. To alleviate this issue, we apply a Gaussian smoothing on $\mathrm{d}f(x_i^{\mathcal{M}})$ around every point $x_i^{\mathcal{M}}$ with the smoothing parameter set to be the side length of the image divided by 50. With the smoothing, we can ensure that the descent is in the right direction while there will not be an extremely large change at a point relative to the local neighborhood of it.

\subsection{The choice of the pre-trained network and the model parameters}
In our experiment, we use a well-known CNN \emph{DenseNet-201}~\cite{huang2017densely}, which is a densely connected convolutional network with 201 layers. All layers after the third dense block are truncated. The images are partitioned into $m = 10 \times 10$, $12 \times 12$, $14 \times 14$, $16 \times 16$ or $18 \times 18$ patches. For each registration, we consider all choices of $m$ and choose the one that gives the best registration performance in terms of the similarity of the warped image and the target image and the smoothness of the mapping. Also, in practice we keep only the top 6 to 50 values of the correlation matrix $C$ depending on the size of the features and set all other values to be 0.

As for the weighting factors in the proposed energy model~\eqref{splitted energy}, in general we set $\alpha = 5$, $\rho = 50$, $\beta = 25\rho$, and $\gamma = 5\rho$. The parameter in Equation~\eqref{eq: D def} is set to be $\sigma =1$. 

\subsection{Multiresolution scheme}
To reduce the computational cost of registering high resolution images, we adopt a multiresolution scheme for the registration procedure. We first coarsen both the moving image $\mathcal{M}$ and the static image $\mathcal{S}$ by $k$ layers, where $I_j^0 = I_j$ is the densest image and $I_j^k$ is the coarsest image of $I_j$ ($j = 1,2$). The registration process starts with registering $\mathcal{M}^k$ and $\mathcal{S}^k$. After obtaining a diffeomorphism $f_k$, we proceed to a finer scale. Specifically, we obtain $f_{k-1}$ by a linear interpolation on $f_k$, which serves as the initial map for the registration at the next finer layer. We repeat the above process until obtaining the registration at the finest (original resolution) layer. Using this multiresolution scheme, the computation can be significantly accelerated.

\section{Experimental results}\label{sect:experiment}
To demonstrate the effectiveness of our proposed method for large deformation image registration, we test it on various synthetic and real medical images. We remark that in all examples below, we focus on the energy at the finest (original resolution) layer of the multiresolution scheme. We compare our method with DDemons~\cite{demons}, LDDMM~\cite{lddmm}, Elastix~\cite{klein2009elastix} and DROP~\cite{glocker2011deformable} (all with code or software available online~\cite{ddemons_code,lddmm_code,elastix_code,drop_code}).

\begin{figure}[t!]
    \centering
    \begin{subfigure}[h]{0.32\textwidth}
        \centering
        \includegraphics[width=\textwidth]{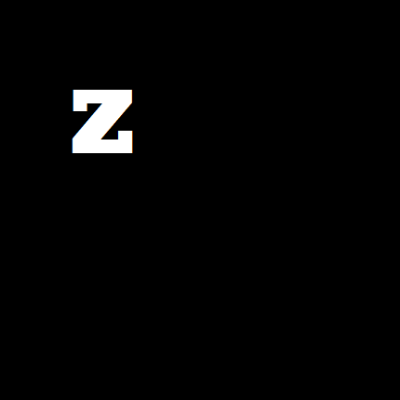}
        \caption{Source image}
        \label{fig: Z to 2 source image}
    \end{subfigure}
    \hfill
    \begin{subfigure}[h]{0.32\textwidth}
        \centering
        \includegraphics[width=\textwidth]{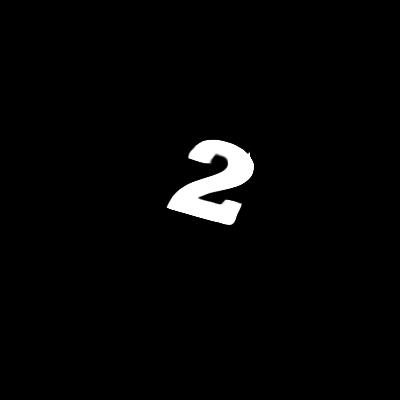}
        \caption{Target image}
        \label{fig: Z to 2 target image}
    \end{subfigure}
    \hfill
    \begin{subfigure}[h]{0.32\textwidth}
        \centering
        \includegraphics[width=\textwidth]{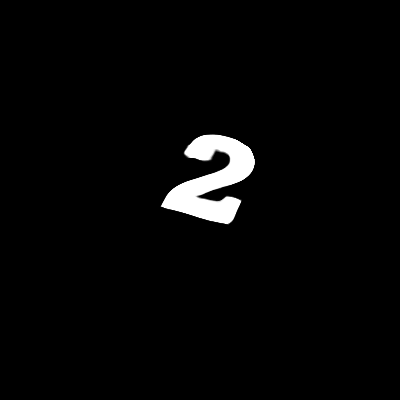}
        \caption{Our registration result}
        \label{fig: Z to 2 warped image}
    \end{subfigure}
    \hfill
    \begin{subfigure}[h]{0.32\textwidth}
        \centering
        \includegraphics[width=\textwidth]{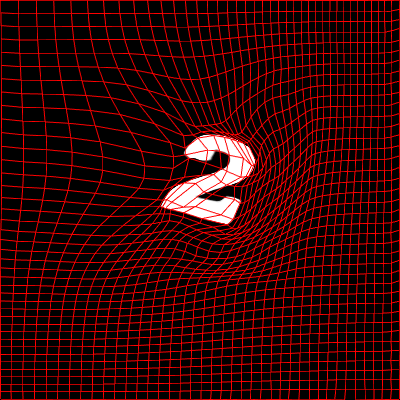}
        \caption{Warped image with grid}
        \label{fig: Z to 2 warped image with grid}
    \end{subfigure}
    \begin{subfigure}[h]{0.32\textwidth}
        \centering
        \includegraphics[width=\textwidth]{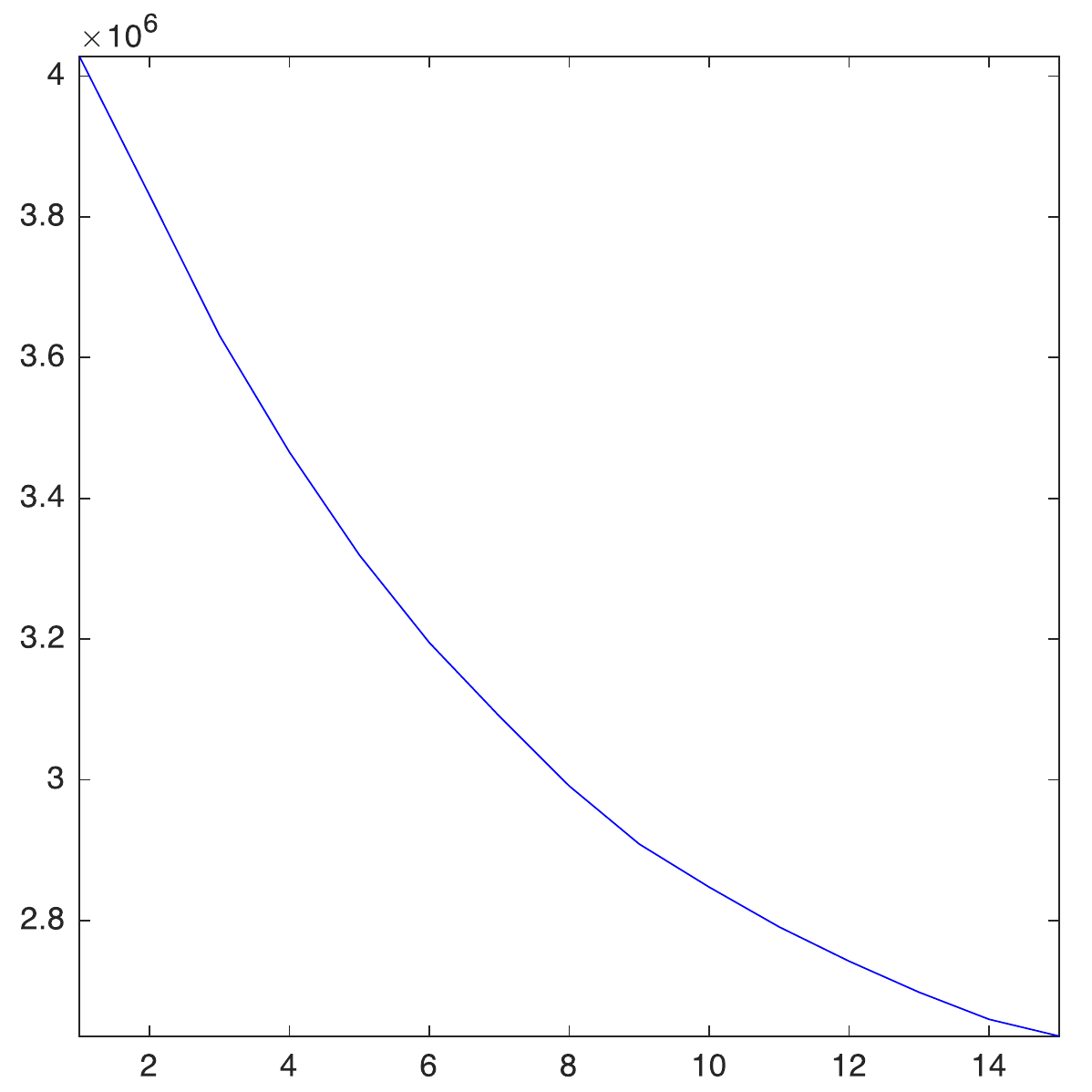}
        \caption{Energy plot}
        \label{fig: Z to 2 energy}
    \end{subfigure}
    \hfill
    \begin{subfigure}[h]{0.32\textwidth}
        \centering
        \includegraphics[width=\textwidth]{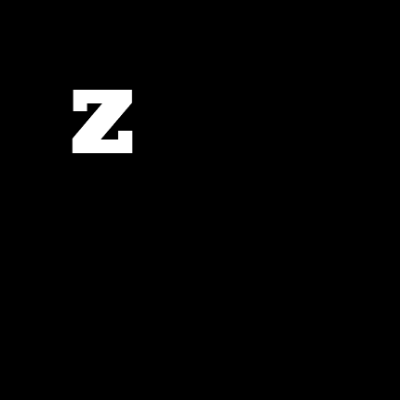}
        \caption{Result by DDemons~\cite{demons}}
        \label{fig: Z to 2 demon result}
    \end{subfigure}
    \hfill
    \begin{subfigure}[h]{0.32\textwidth}
        \centering
        \includegraphics[width=\textwidth]{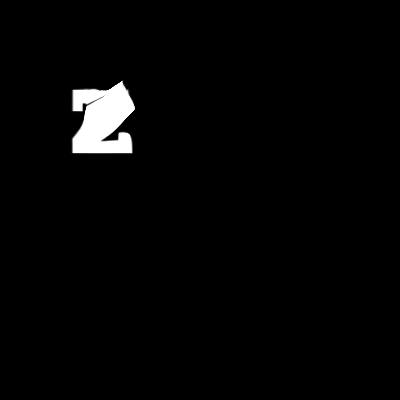}
        \caption{Result by LDDMM~\cite{lddmm}}
        \label{fig: Z to 2 lddmm result}
    \end{subfigure}
    \hfill
    \begin{subfigure}[h]{0.32\textwidth}
        \centering
        \includegraphics[width=\textwidth]{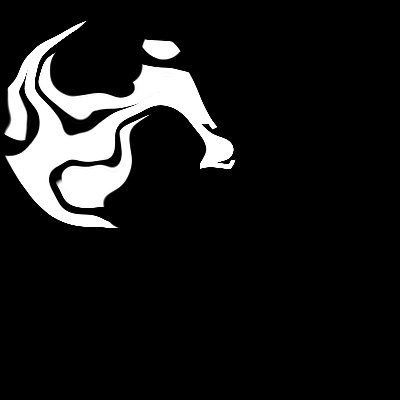}
        \caption{Result by Elastix~\cite{klein2009elastix}}
        \label{fig: Z to 2 elastix result}
    \end{subfigure}
    \hfill
    \begin{subfigure}[h]{0.32\textwidth}
        \centering
        \includegraphics[width=\textwidth]{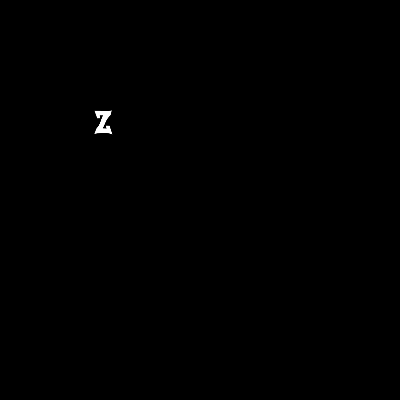}
        \caption{Result by DROP~\cite{glocker2011deformable}}
        \label{fig: Z to 2 drop result}
    \end{subfigure}
    \caption{The `Z' to `2' example.}
    \label{fig:synthetic}
\end{figure}

\begin{figure}[t!]
    \centering
    \begin{subfigure}{0.32\textwidth}
        \includegraphics[width=\textwidth]{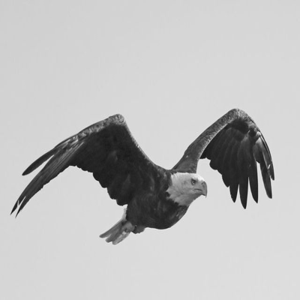}
        \caption{Source image}
        \label{fig: eagle source image}
    \end{subfigure}
    \hfill
    \begin{subfigure}{0.32\textwidth}
        \includegraphics[width=\textwidth]{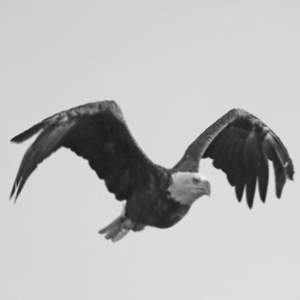}
        \caption{Target image}
        \label{fig: eagle target image}
    \end{subfigure}
    \hfill
    \begin{subfigure}{0.32\textwidth}
        \includegraphics[width=\textwidth]{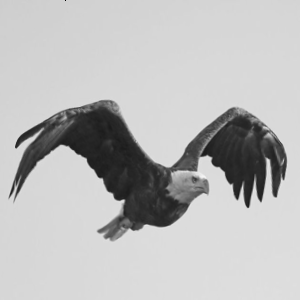}
        \caption{Our registration result}
        \label{fig: eagle warped image}
    \end{subfigure}
    \hfill
    \begin{subfigure}{0.32\textwidth}
        \includegraphics[width=\textwidth]{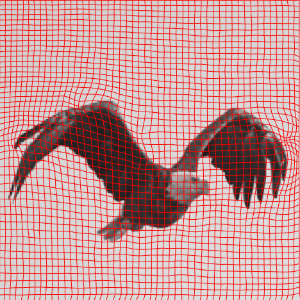}
        \caption{Warped image with grid}
        \label{fig: eagle warped image with grid}
    \end{subfigure}
    \hfill
    \begin{subfigure}{0.32\textwidth}
        \includegraphics[width=\textwidth]{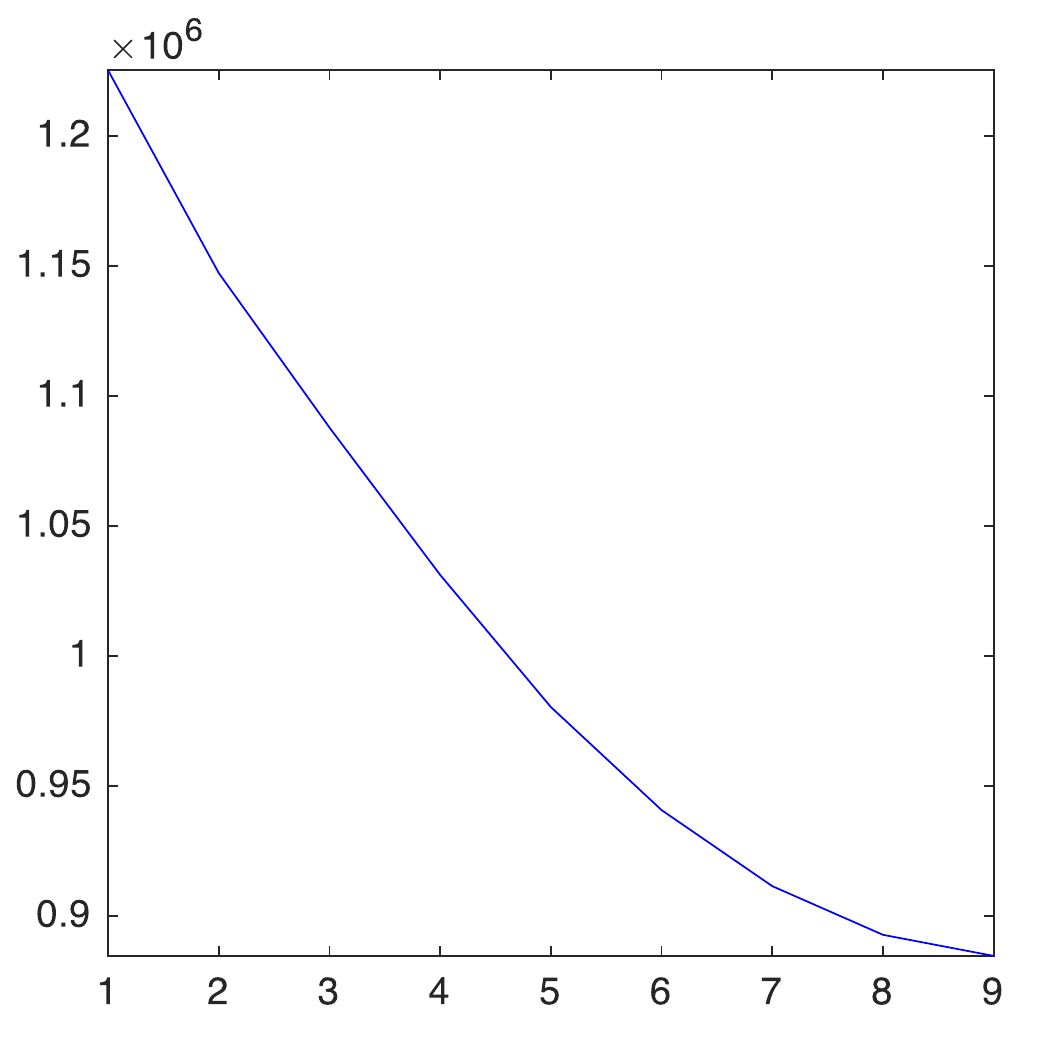}
        \caption{Energy plot}
        \label{fig: eagle energy}
    \end{subfigure}
    \hfill
    \begin{subfigure}{0.32\textwidth}
        \includegraphics[width=\textwidth]{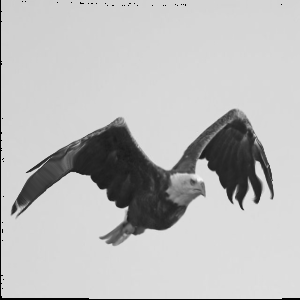}
        \caption{Result by LDDMM~\cite{lddmm}}
        \label{fig: eagle lddmm}
    \end{subfigure}
    \hfill
    \begin{subfigure}{0.32\textwidth}
        \includegraphics[width=\textwidth]{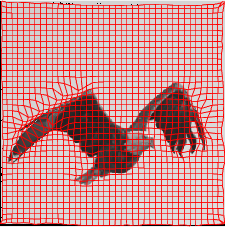}
        \caption{LDDMM deformed grid}
        \label{fig: eagle lddmm grid}
    \end{subfigure}
    \hfill
    \begin{subfigure}{0.32\textwidth}
        \includegraphics[width=\textwidth]{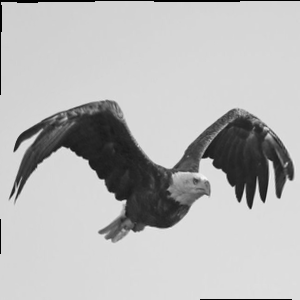}
        \caption{Result by Elastix~\cite{klein2009elastix}}
        \label{fig: eagle elastix}
    \end{subfigure}
    \hfill
    \begin{subfigure}{0.32\textwidth}
        \includegraphics[width=\textwidth]{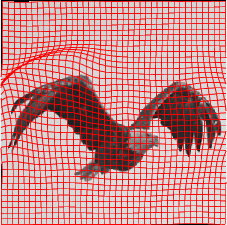}
        \caption{Elastix deformed grid}
        \label{fig: eagle elastix grid}
    \end{subfigure}
    \caption{The eagle example.}
    \label{fig:eagle}
\end{figure}

\begin{figure}[t!]
    \centering
    \begin{subfigure}{0.32\textwidth}
        \includegraphics[width=\textwidth]{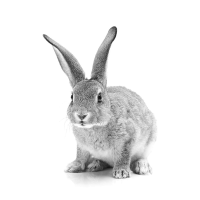}
        \caption{Source image}
        \label{fig: rabbit source image}
    \end{subfigure}
    \hfill
    \begin{subfigure}{0.32\textwidth}
        \includegraphics[width=\textwidth]{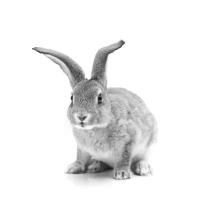}
        \caption{Target image}
        \label{fig: rabbit target image}
    \end{subfigure}
    \hfill
    \begin{subfigure}{0.32\textwidth}
        \includegraphics[width=\textwidth]{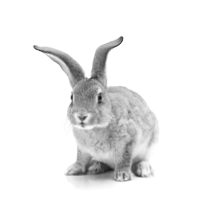}
        \caption{Our registration result}
        \label{fig: rabbit warped image}
    \end{subfigure}
    \hfill
    \begin{subfigure}{0.32\textwidth}
        \includegraphics[width=\textwidth]{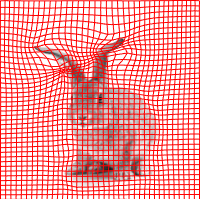}
        \caption{Warped image with grid}
        \label{fig: rabbit warped image with grid}
    \end{subfigure}
    \hfill
    \begin{subfigure}{0.32\textwidth}
        \includegraphics[width=\textwidth]{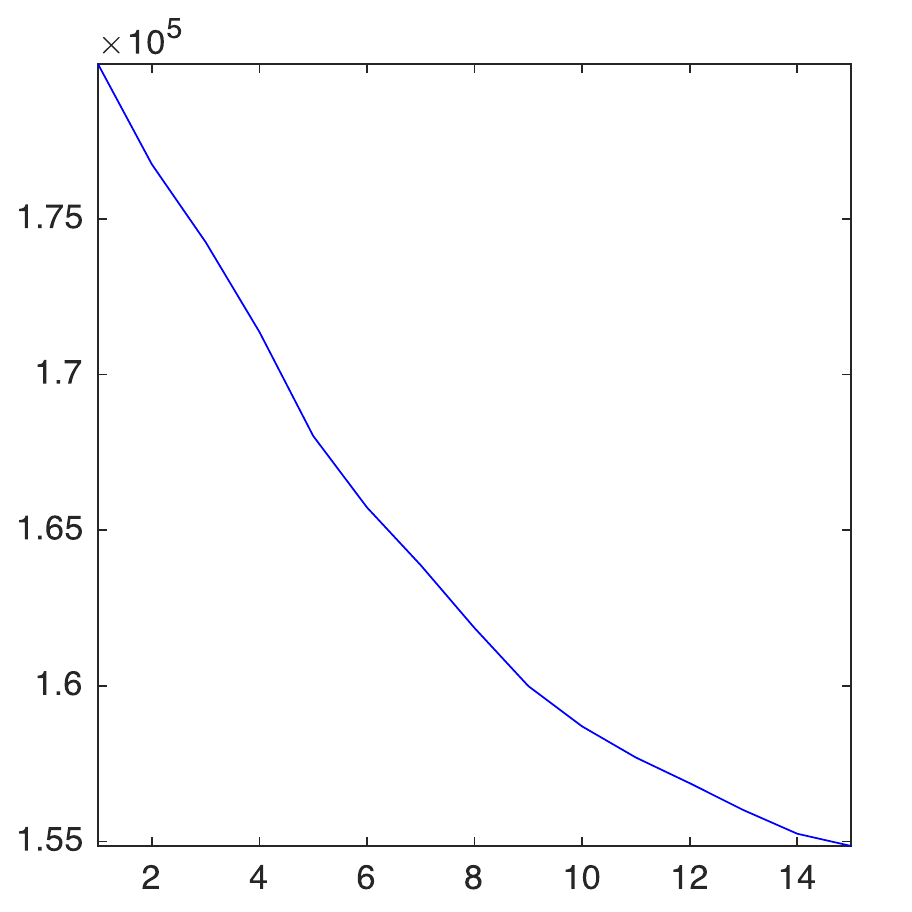}
        \caption{Energy plot}
        \label{fig: rabbit energy}
    \end{subfigure}
    \hfill
    \begin{subfigure}{0.32\textwidth}
        \includegraphics[width=\textwidth]{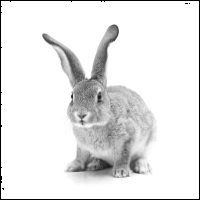}
        \caption{Result by LDDMM~\cite{lddmm}}
        \label{fig: rabbit lddmm}
    \end{subfigure}
    \hfill
    \begin{subfigure}{0.32\textwidth}
        \includegraphics[width=\textwidth]{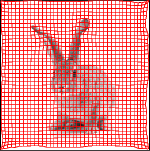}
        \caption{LDDMM deformed grid}
        \label{fig: rabbit lddmm grid}
    \end{subfigure}
    \hfill
    \begin{subfigure}{0.32\textwidth}
        \includegraphics[width=\textwidth]{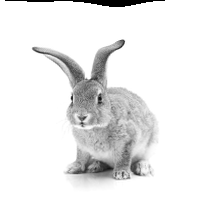}
        \caption{Result by Elastix~\cite{klein2009elastix}}
        \label{fig: rabbit elastix}
    \end{subfigure}
    \hfill
    \begin{subfigure}{0.32\textwidth}
        \includegraphics[width=\textwidth]{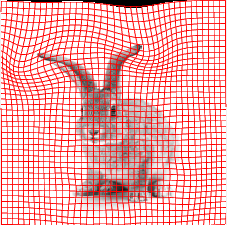}
        \caption{Elastix deformed grid}
        \label{fig: rabbit elastix grid}
    \end{subfigure}
    \caption{The rabbit example.}
    \label{fig:rabbit}
\end{figure}

\subsection{Synthetic images}
We first test our method on a synthetic example with the source image containing a letter `Z' (Fig.~\ref{fig: Z to 2 source image}) and the target image containing a tilted number `2' (Fig.~\ref{fig: Z to 2 target image}). Two major difficulties of this example are the large displacement between the `Z' and the `2' and the shape difference between them involving certain sharp angle changes. Although one may use an affine map to roughly align them, the boundary of the two images will be largely mismatched. As shown in Fig.~\ref{fig: Z to 2 warped image}, our proposed registration method effectively morphs the `Z' image to match the `2' image. One can also visualize the registration via the deformed underlying grid as shown in Fig.~\ref{fig: Z to 2 warped image with grid}, from which it can be observed that the mapping is smooth and bijective. As shown in Fig.~\ref{fig: Z to 2 energy}, the energy decreases rapidly throughout the iterations. For comparison, we apply DDemons~\cite{demons} (Fig.~\ref{fig: Z to 2 demon result}), LDDMM~\cite{lddmm} (Fig.~\ref{fig: Z to 2 lddmm result}), Elastix~\cite{klein2009elastix} (Fig.~\ref{fig: Z to 2 elastix result}), and DROP~\cite{glocker2011deformable} (Fig.~\ref{fig: Z to 2 drop result}) for registering the images. It can be observed that all these methods fail to register the `Z' and the `2' shape due to the large displacement and shape change.

We then consider another synthetic example with the source image being an eagle (Fig.~\ref{fig: eagle source image}). For the target image (Fig.~\ref{fig: eagle target image}), we manually deform the eagle shape such that the wings are expanded wider. We apply our proposed method to obtain a registration between the two images (see Fig.~\ref{fig: eagle warped image} for the warped image and Fig.~\ref{fig: eagle warped image with grid} for the warped image with the deformed underlying grid). Under the registration, the difference between the two images is significantly reduced, and the energy plot in Fig.~\ref{fig: eagle energy} again shows that our method is highly effective. On the contrary, it can be observed that LDDMM~\cite{lddmm} does not produce an accurate registration of the wings (Fig.~\ref{fig: eagle lddmm} and Fig.~\ref{fig: eagle lddmm grid}). While the registration result by Elastix~\cite{klein2009elastix} looks satisfactory (Fig.~\ref{fig: eagle elastix}), there are multiple overlaps in the deformed underlying grid (Fig.~\ref{fig: eagle elastix grid}). Overall, our method gives the best registration result with a more natural positional correspondence between different parts as shown by the deformation of the underlying grid and with bijectivity ensured.

In the next example, the source image is a rabbit (Fig.~\ref{fig: rabbit source image}), and the target image is the same rabbit with its ears bent (Fig.~\ref{fig: rabbit target image}). Our method successfully morphs the ears to the desired position as shown in Fig.~\ref{fig: rabbit warped image}. From the deformed underlying grid in Fig.~\ref{fig: rabbit warped image with grid}, we can see that the mapping is smooth and bijective. Also, the energy plot in Fig.~\ref{fig: rabbit energy} shows that our method effectively reduces the energy within a small number of iterations. On the contrary, LDDMM~\cite{lddmm} fails to match the bent ears (Fig.~\ref{fig: rabbit lddmm} and Fig.~\ref{fig: rabbit lddmm grid}). While Elastix~\cite{klein2009elastix} is able to deform the ears, the overall image shape is largely distorted (Fig.~\ref{fig: rabbit elastix} and Fig.~\ref{fig: rabbit elastix grid}).

\begin{figure}[t!]
    \centering
    \begin{subfigure}[h]{0.32\textwidth}
        \centering
        \includegraphics[width=\textwidth]{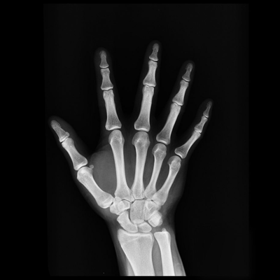}
        \caption{Source image}
        \label{fig: hand source image}
    \end{subfigure}
    \hfill
    \begin{subfigure}[h]{0.32\textwidth}
        \centering
        \includegraphics[width=\textwidth]{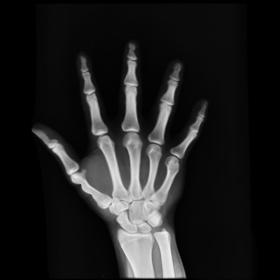}
        \caption{Target image}
        \label{fig: hand target image}
    \end{subfigure}
    \hfill
    \begin{subfigure}[h]{0.32\textwidth}
        \centering
        \includegraphics[width=\textwidth]{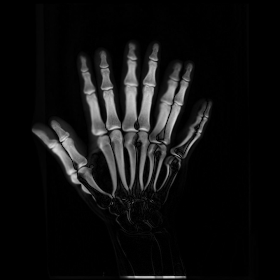}
        \caption{Intensity difference}
        \label{fig: hand initial diff}
    \end{subfigure}
    \hfill
    \begin{subfigure}[h]{0.32\textwidth}
        \centering
        \includegraphics[width=\textwidth]{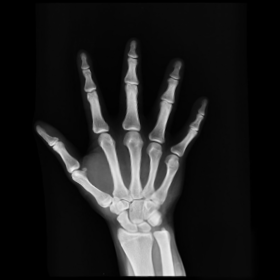}
        \caption{Our registration result}
        \label{fig: hand warped image}
    \end{subfigure}
    \hfill
    \begin{subfigure}[h]{0.32\textwidth}
        \centering
        \includegraphics[width=\textwidth]{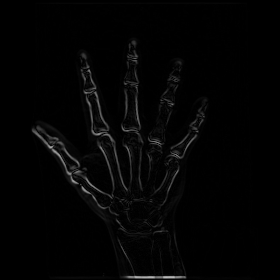}
        \caption{Final intensity difference}
        \label{fig: hand final diff}
    \end{subfigure}
    \hfill
    \begin{subfigure}[h]{0.32\textwidth}
        \centering
        \includegraphics[width=\textwidth]{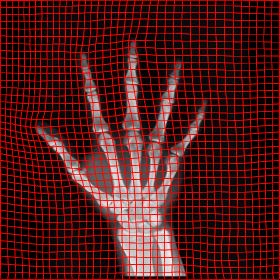}
        \caption{Warped image with grid}
        \label{fig: hand warped image with grid}
    \end{subfigure}
    \hfill
    \begin{subfigure}{0.32\textwidth}
        \includegraphics[width=\textwidth]{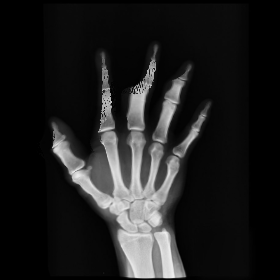}
        \caption{Result by LDDMM~\cite{lddmm}}
        \label{fig: hand lddmm}
    \end{subfigure}
    \hfill
    \begin{subfigure}{0.32\textwidth}
        \includegraphics[width=\textwidth]{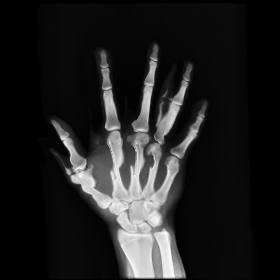}
        \caption{Result by DDemons~\cite{demons}}
        \label{fig: hand ddemons}
    \end{subfigure}
    \hfill
    \begin{subfigure}{0.32\textwidth}
        \includegraphics[width=\textwidth]{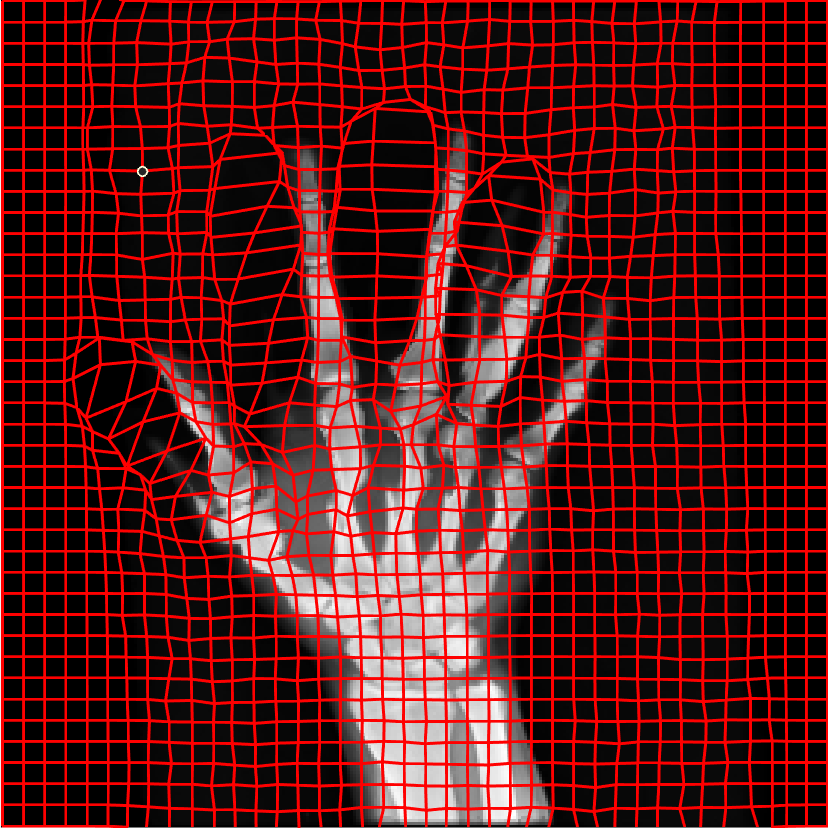}
        \caption{DDemons deformed grid}
        \label{fig: hand ddemons grid}
    \end{subfigure}
    \caption{The first hand X-ray example.}
    \label{fig:hand1}
\end{figure}

\subsection{Real X-ray images}
After testing our proposed method on several synthetic images, we now consider applying it on real medical images. Here, we consider a hand X-ray image as the source image (Fig.~\ref{fig: hand source image}) and a deformed hand X-ray image as the target image (Fig.~\ref{fig: hand target image}). Fig.~\ref{fig: hand initial diff} shows the original absolute intensity difference between the two images. It can be observed that different fingers are displaced in a nonuniform manner (for example, the displacement of the index finger is much larger than that of the little finger), while the wrist remains almost the same. Therefore, a simple rigid transformation is insufficient for yielding a good registration. As shown in Fig.~\ref{fig: hand warped image}, our proposed method successfully deforms the source image to match the target image, and the final intensity difference is significantly smaller (see Fig.~\ref{fig: hand final diff}). From the deformed underlying grid in Fig.~\ref{fig: hand warped image with grid}, it can be observed that the mapping is smooth and bijective. For comparison, both LDDMM~\cite{lddmm} and DDemons~\cite{demons} fail to register the fingers and are non-bijective (see Fig.~\ref{fig: hand lddmm}, Fig.~\ref{fig: hand ddemons}, and Fig.~\ref{fig: hand ddemons grid}). 

We then consider another example of registering two hand X-ray images with larger deformations (see Fig.~\ref{fig: second hand source image} for the source image, Fig.~\ref{fig: second hand target image} for the target image, and Fig.~\ref{fig: second hand initial diff} for their absolute intensity difference). The warped image produced by our proposed method (Fig.~\ref{fig: second hand warped image}) again closely resembles the target image with the intensity difference significantly reduced (see Fig.~\ref{fig: second hand final diff}). Fig.~\ref{fig: second hand warped image with grid} shows that the mapping is smooth and bijective. For comparison, note that LDDMM~\cite{lddmm} fails to match the fingers (Fig.~\ref{fig: second hand lddmm}). While DROP~\cite{glocker2011deformable} is capable of registering the fingers (Fig.~\ref{fig: second hand drop}), it distorts the boundary shape of the overall image (Fig.~\ref{fig: second hand drop grid}).

\begin{figure}[t!]
    \centering
    \begin{subfigure}[h]{0.32\textwidth}
        \centering
        \includegraphics[width=\textwidth]{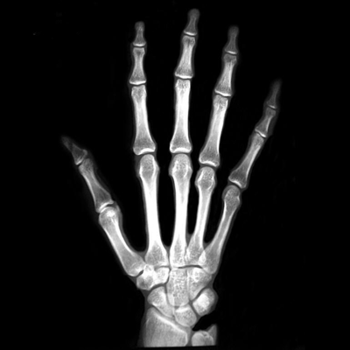}
        \caption{Source image}
        \label{fig: second hand source image}
    \end{subfigure}
    \hfill
    \begin{subfigure}[h]{0.32\textwidth}
        \centering
        \includegraphics[width=\textwidth]{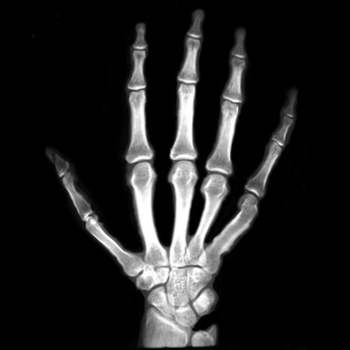}
        \caption{Target image}
        \label{fig: second hand target image}
    \end{subfigure}
    \hfill
    \begin{subfigure}[h]{0.32\textwidth}
        \centering
        \includegraphics[width=\textwidth]{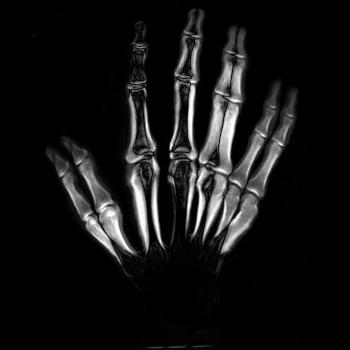}
        \caption{Intensity difference}
        \label{fig: second hand initial diff}
    \end{subfigure}
    \hfill
    \begin{subfigure}[h]{0.32\textwidth}
        \centering
        \includegraphics[width=\textwidth]{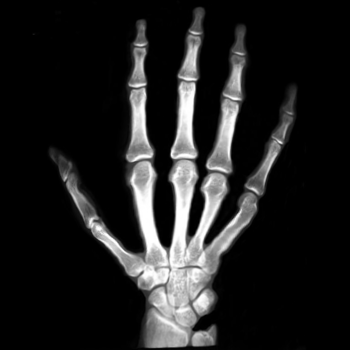}
        \caption{Our registration result}
        \label{fig: second hand warped image}
    \end{subfigure}
    \hfill
    \begin{subfigure}[h]{0.32\textwidth}
        \centering
        \includegraphics[width=\textwidth]{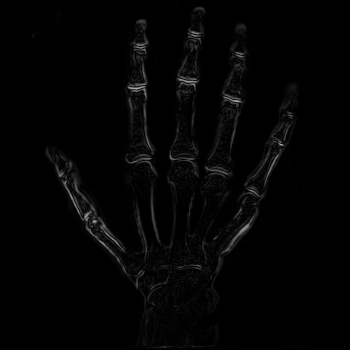}
        \caption{Final intensity difference}
        \label{fig: second hand final diff}
    \end{subfigure}
    \hfill
    \begin{subfigure}[h]{0.32\textwidth}
        \centering
        \includegraphics[width=\textwidth]{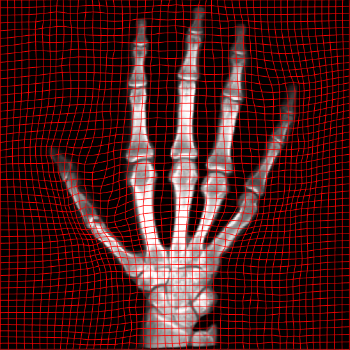}
        \caption{Warped image with grid}
        \label{fig: second hand warped image with grid}
    \end{subfigure}
    \hfill
    \begin{subfigure}{0.32\textwidth}
        \includegraphics[width=\textwidth]{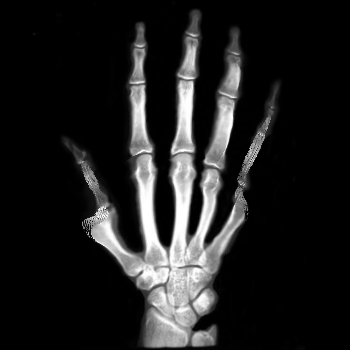}
        \caption{Result by LDDMM~\cite{lddmm}}
        \label{fig: second hand lddmm}
    \end{subfigure}
    \hfill
   \begin{subfigure}{0.32\textwidth}
        \includegraphics[width=\textwidth]{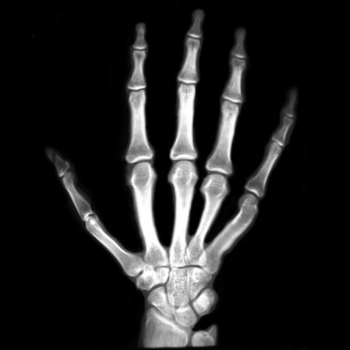}
        \caption{Result by DROP~\cite{glocker2011deformable}}
        \label{fig: second hand drop}
    \end{subfigure}
    \hfill
   \begin{subfigure}{0.32\textwidth}
        \includegraphics[width=\textwidth]{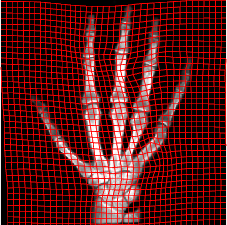}
        \caption{DROP deformed grid}
        \label{fig: second hand drop grid}
    \end{subfigure}
    \caption{The second hand X-ray example.}
    \label{fig:hand2}
\end{figure}

\begin{figure}[t!]
    \centering
    \begin{subfigure}[h]{0.32\textwidth}
        \centering
        \includegraphics[width=\textwidth]{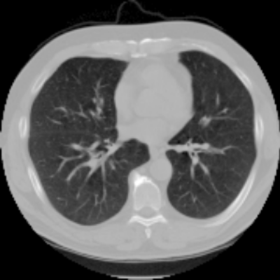}
        \caption{Source image}
        \label{fig: first lung source image}
    \end{subfigure}
    \hfill
    \begin{subfigure}[h]{0.32\textwidth}
        \centering
        \includegraphics[width=\textwidth]{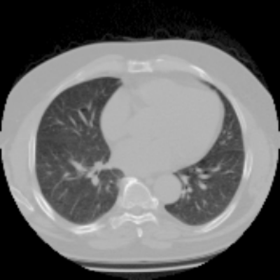}
        \caption{Target image}
        \label{fig: first lung target image}
    \end{subfigure}
    \hfill
    \begin{subfigure}[h]{0.32\textwidth}
        \centering
        \includegraphics[width=\textwidth]{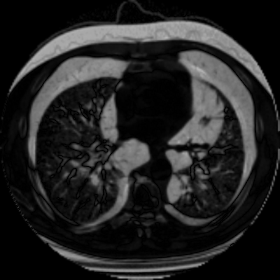}
        \caption{Intensity difference}
        \label{fig: first lung initial diff}
    \end{subfigure}
    \hfill
    \begin{subfigure}[h]{0.32\textwidth}
        \centering
        \includegraphics[width=\textwidth]{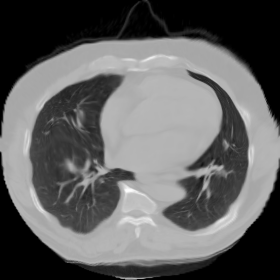}
        \caption{Our registration result}
        \label{fig: first lung warped image}
    \end{subfigure}
    \hfill
    \begin{subfigure}[h]{0.32\textwidth}
        \centering
        \includegraphics[width=\textwidth]{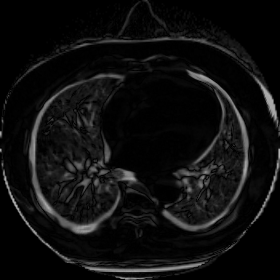}
        \caption{Final intensity difference}
        \label{fig: first lung final diff}
    \end{subfigure}
    \hfill
    \begin{subfigure}[h]{0.32\textwidth}
        \centering
        \includegraphics[width=\textwidth]{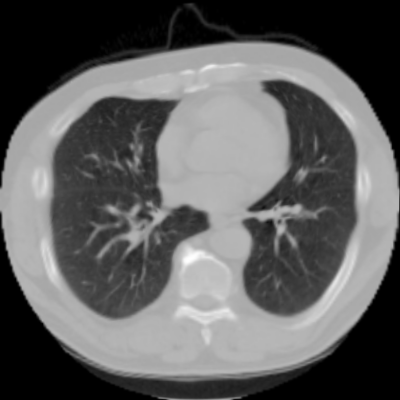}
        \caption{Result by DDemons~\cite{demons}}
        \label{fig: first lung demon result}
    \end{subfigure}
    \hfill
    \begin{subfigure}[h]{0.32\textwidth}
        \centering
        \includegraphics[width=\textwidth]{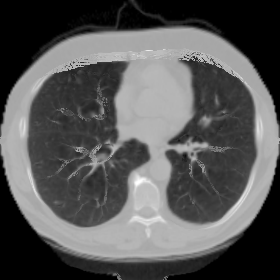}
        \caption{Result by LDDMM~\cite{lddmm}}
        \label{fig: first lung lddmm result}
    \end{subfigure}
    \hfill
    \begin{subfigure}[h]{0.32\textwidth}
        \centering
        \includegraphics[width=\textwidth]{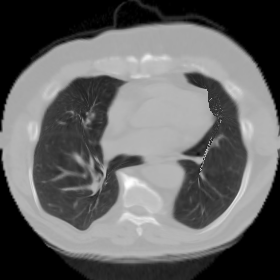}
        \caption{Result by Elastix~\cite{klein2009elastix}}
        \label{fig: first lung elastix result}
    \end{subfigure}
    \hfill
    \begin{subfigure}[h]{0.32\textwidth}
        \centering
        \includegraphics[width=\textwidth]{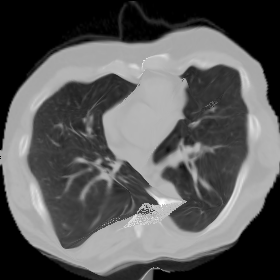}
        \caption{Result by DROP~\cite{glocker2011deformable}}
        \label{fig: first lung drop result}
    \end{subfigure}
    \caption{The lung CT example.}
    \label{fig:lung1}
\end{figure}

\begin{figure}[t!]
    \centering
    \begin{subfigure}[h]{0.32\textwidth}
        \centering
        \includegraphics[width=\textwidth]{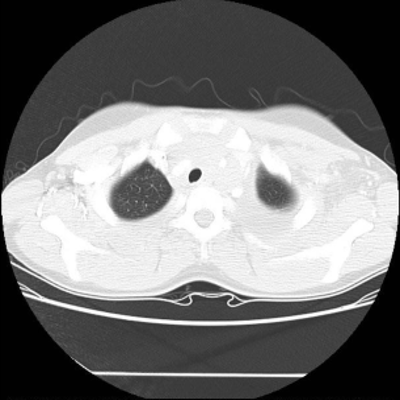}
        \caption{Source image}
        \label{fig: second lung source image}
    \end{subfigure}
    \begin{subfigure}[h]{0.32\textwidth}
        \centering
        \includegraphics[width=\textwidth]{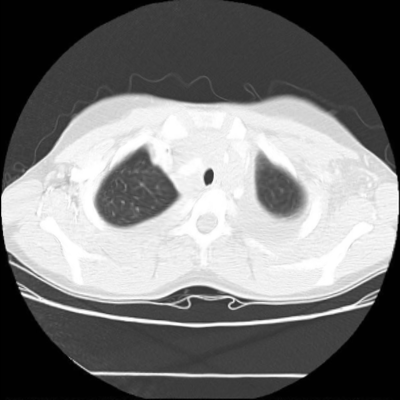}
        \caption{Target image}
        \label{fig: second lung target image}
    \end{subfigure}
    \begin{subfigure}[h]{0.32\textwidth}
        \centering
        \includegraphics[width=\textwidth]{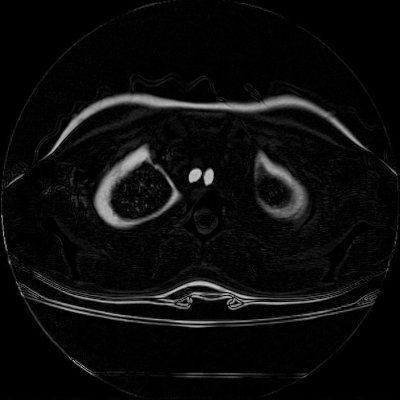}
        \caption{Intensity difference}
        \label{fig: second lung initial diff}
    \end{subfigure}
    \begin{subfigure}[h]{0.32\textwidth}
        \centering
        \includegraphics[width=\textwidth]{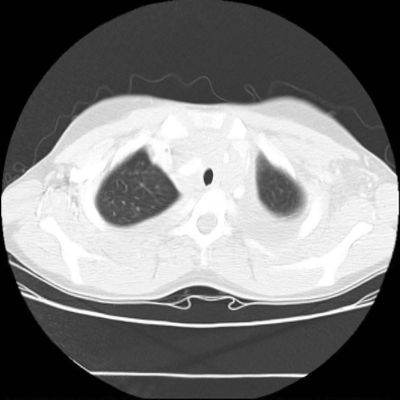}
        \caption{Our registration result}
        \label{fig: second lung warped image}
    \end{subfigure}
    \begin{subfigure}[h]{0.32\textwidth}
        \centering
        \includegraphics[width=\textwidth]{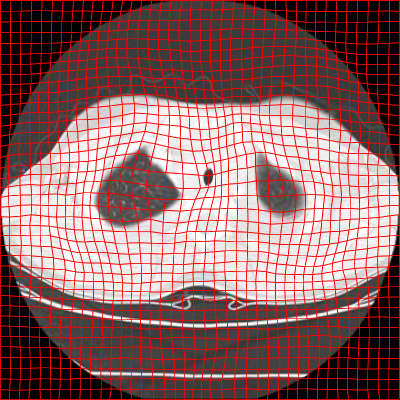}
        \caption{Warped image with grid}
        \label{fig: second lung warped image with grid}
    \end{subfigure}
    \begin{subfigure}[h]{0.32\textwidth}
        \centering
        \includegraphics[width=\textwidth]{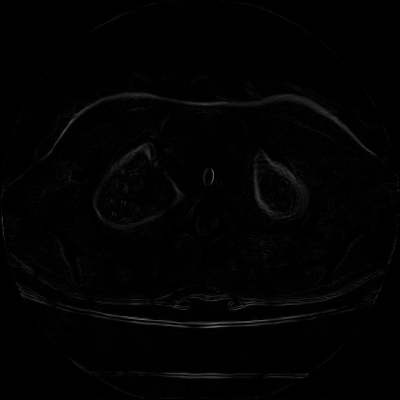}
        \caption{Final intensity difference}
        \label{fig: second lung final diff}
    \end{subfigure}
    \begin{subfigure}[h]{0.32\textwidth}
        \centering
        \includegraphics[width=\textwidth]{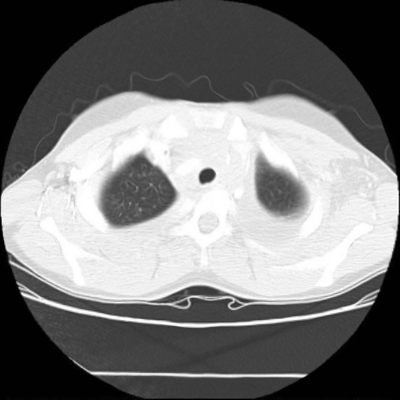}
        \caption{The result by DDemons~\cite{demons}}
        \label{fig: second lung demon}
    \end{subfigure}
    \begin{subfigure}[h]{0.32\textwidth}
        \centering
        \includegraphics[width=\textwidth]{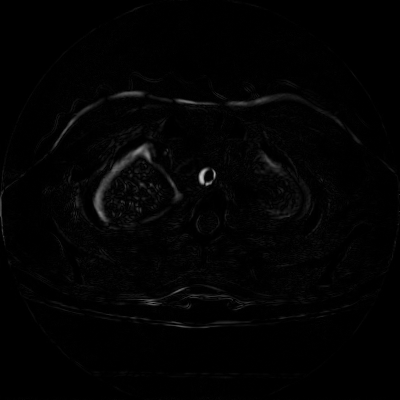}
        \caption{DDemons intensity difference}
        \label{fig: second lung demon absolute diff}
    \end{subfigure}
    \caption{The chest CT example.}
    \label{fig:lung2}
\end{figure}

\subsection{Real CT images}
We now consider applying the proposed image registration method on real lung CT images retrieved from the National Lung Screening Trial (NLST) dataset~\cite{nlst}. Fig.~\ref{fig: first lung source image} and Fig.~\ref{fig: first lung target image} show two slices of lung CT images that we use as the source and the target (see Fig.~\ref{fig: first lung initial diff} for the absolute intensity difference). We remark that the CT images are originally with different intensity, and so we apply an intensity histogram matching before running the registration experiment. Fig.~\ref{fig: first lung warped image} shows the registration result obtained by our proposed method. It can be observed that our method successfully produces a large deformation on the right lung of the source image to match that of the target image (see also Fig.~\ref{fig: first lung final diff} for the final absolute intensity difference). On the contrary, DDemons~\cite{demons} (Fig.~\ref{fig: first lung demon result}), LDDMM~\cite{lddmm} (Fig.~\ref{fig: first lung lddmm result}), Elastix~\cite{klein2009elastix}  (Fig.~\ref{fig: first lung elastix result}) and DROP~\cite{glocker2011deformable} (Fig.~\ref{fig: first lung drop result}) all fail to produce an accurate and bijective registration result. This shows that our method is more capable of handling large deformation image registration. 

We then test our method on slices of chest CT images obtained from the Open Access Biomedical Image Search Engine~\cite{openi}. Fig.~\ref{fig: second lung source image} and~\ref{fig: second lung target image} show the source image and target image respectively, and the intensity difference is shown in Fig.~\ref{fig: second lung initial diff}. The registration result obtained by our proposed method is shown in Fig.~\ref{fig: second lung warped image} (see also the result with the deformed underlying grid in Fig.~\ref{fig: second lung warped image with grid}). From the final intensity difference plot in Fig.~\ref{fig: second lung final diff}, it is easily to see that our method matches not only the two large components but also the small dot at the center very well. On the contrary, DDemons~\cite{demons} produces a suboptimal registration result with a significantly larger mismatch of the small component at the center (see Fig.~\ref{fig: second lung demon} and Fig.~\ref{fig: second lung demon absolute diff}).

\subsection{Quantitative comparison between different image registration methods}
For a more quantitative analysis, we compare our proposed image registration method with DDemons~\cite{demons}, LDDMM~\cite{lddmm}, Elastix~\cite{klein2009elastix}, and DROP~\cite{glocker2011deformable} in terms of the accuracy, smoothness, overall quality, and bijectivity. Here, to assess the accuracy of a registration map $f: \mathcal{M} \to \mathcal{S}$, we consider the following similarity energy
\begin{equation}\label{eqt:esim}
\begin{split}
    E_{\text{sim}}(f) = & \frac{\sum \sum |I_{\mathcal{M}}-I_{\mathcal{S}}(f)|}{2}\left(\frac{1}{\sum \sum I_{\mathcal{M}}} + \frac{1}{\sum \sum (1-I_{\mathcal{M}})} \right. \\
    & \left. + \frac{1}{\sum \sum I_{\mathcal{S}}(f)} + \frac{1}{\sum \sum (1-I_{\mathcal{S}}(f))}\right).
\end{split}
\end{equation}
Note that $E_{\text{sim}} = 0$ if and only if $f(\mathcal{M}) = \mathcal{S}$, i.e. the registration matches the two images exactly. To assess the smoothness of $f$, we consider the following smoothness energy
\begin{equation}\label{eqt:esmooth}
    E_{\text{smooth}}(f) = \frac{1}{mn} \sqrt{\sum \sum \left(\left(\frac{\partial f}{\partial x}\right)^2 + \left(\frac{\partial f}{\partial y}\right)^2 \right)},
\end{equation}
where the size of $\mathcal{M}$ is $mn$. The total energy $E_{\text{total}}$ takes both the similarity and the smoothness into account for evaluating the overall quality of the registration:
\begin{equation}\label{eqt:etotal}
    E_{\text{total}} = E_{\text{sim}} + E_{\text{smooth}}.
\end{equation}
Lastly, the bijectivity can be assessed by counting the number of flips in the mapping result. More specifically, we divide every image into $h \times w$ small squares evenly, and further divide each small square into 2 triangles. After applying each registration method, we count the number of triangles with the orientation flipped. A registration is bijective if and only if there is no triangle flip.

Table~\ref{table: comparison} records the performance of the above-mentioned methods. It can be observed that our method always gives a bijective registration result, while the other methods may produce flips and hence are non-bijective for various examples. Also, our method achieves the lowest total energy among all methods for most examples. For the few examples that some other methods are with a lower total energy, their results all contain a large number of flips and hence are still undesirable. The experiments suggest that our proposed method is more desirable for large deformation image registration, with the bijectivity guaranteed and a good overall quality in terms of accuracy and smoothness achieved.

\begin{table}[t!]
\centering
\resizebox{\textwidth}{!}{%
\begin{tabular}{|c|c|c|c|c|c|c|c|}
\hline
\multirow{2}{*}{{\bf Method}} & \multicolumn{7}{c|}{{\bf Results ($E_{\text{sim}}$ / $E_{\text{smooth}}$ / $E_{\text{total}}$ / \# Flips)}} \\ \cline{2-8}
                &  `Z' to `2' (Fig.~\ref{fig:synthetic}) &  Eagle (Fig.~\ref{fig:eagle}) & Rabbit (Fig.~\ref{fig:rabbit}) & Hand 1 (Fig.~\ref{fig:hand1}) & Hand 2 (Fig.~\ref{fig:hand2}) & Lung (Fig.~\ref{fig:lung1}) & Chest (Fig.~\ref{fig:lung2})  \\  \hline
Our method &\begin{tabular}[c]{@{}l@{}}\textbf{0.3099} 0.4476\\ \textbf{0.7575} \textbf{0}\end{tabular} & \begin{tabular}[c]{@{}l@{}}\textbf{0.0916} 0.1558\\ \textbf{0.2474} \textbf{0}\end{tabular}    & \begin{tabular}[c]{@{}l@{}}0.1953 0.1436\\ 0.3389 \textbf{0}\end{tabular}  & \begin{tabular}[c]{@{}l@{}}0.1417 \textbf{0.1075}\\ \textbf{0.2492} \textbf{0}\end{tabular}    & \begin{tabular}[c]{@{}l@{}}\textbf{0.1317} 0.1536\\ \textbf{0.2853} \textbf{0}\end{tabular}    & \begin{tabular}[c]{@{}l@{}}\textbf{0.2435} 0.4774\\ 0.7210 \textbf{0}\end{tabular}    & \begin{tabular}[c]{@{}l@{}}\textbf{0.0368} 0.1414\\ \textbf{0.1783} \textbf{0}\end{tabular}   \\ \hline
DDemons~\cite{demons}         & \begin{tabular}[c]{@{}l@{}}2.0660 \textbf{0.0481}\\ 2.1141 121 \end{tabular} & \begin{tabular}[c]{@{}l@{}}0.3724 0.2470\\ 0.6194 3697\end{tabular}    & \begin{tabular}[c]{@{}l@{}}\textbf{0.1806} 0.1278\\ \textbf{0.3083} 174\end{tabular}  & \begin{tabular}[c]{@{}l@{}}0.4528 0.2651\\ 0.7180 6036\end{tabular}  & \begin{tabular}[c]{@{}l@{}}0.4273 0.2718\\ 0.6991 8925\end{tabular}    & \begin{tabular}[c]{@{}l@{}}0.6332 \textbf{0.2271}\\ 0.8602 4191\end{tabular} & \begin{tabular}[c]{@{}l@{}}0.0725 0.3565\\ 0.4290 16316\end{tabular}  \\ \hline
LDDMM~\cite{lddmm}           & \begin{tabular}[c]{@{}l@{}}2.0488 0.2941\\ 2.3629 105\end{tabular}& \begin{tabular}[c]{@{}l@{}}0.3570 0.2388\\ 0.5968 26\end{tabular} & \begin{tabular}[c]{@{}l@{}}0.5973 \textbf{0.0738}\\ 0.6711 1\end{tabular} & \begin{tabular}[c]{@{}l@{}}0.7018 0.1842\\ 0.8860 \textbf{0}\end{tabular} & \begin{tabular}[c]{@{}l@{}}0.9563 0.1902\\ 1.1465 6\end{tabular} & \begin{tabular}[c]{@{}l@{}}0.6843 0.2568\\ 0.9410 209\end{tabular} & \begin{tabular}[c]{@{}l@{}}0.2446 \textbf{0.0717}\\ 0.3164 \textbf{0}\end{tabular} \\ \hline
Elastix~\cite{klein2009elastix}         & \begin{tabular}[c]{@{}l@{}}3.6377 0.9185\\ 4.5562 68280\end{tabular} & \begin{tabular}[c]{@{}l@{}}0.1324 0.1779\\ 0.3103 1158\end{tabular} & \begin{tabular}[c]{@{}l@{}}0.1808 0.1679\\ 0.3487 \textbf{0}\end{tabular}  & \begin{tabular}[c]{@{}l@{}}\textbf{0.1103} 0.1513\\ 0.2617 \textbf{0}\end{tabular}    & \begin{tabular}[c]{@{}l@{}}0.1856 \textbf{0.1406}\\ 0.3263 \textbf{0}\end{tabular}    & \begin{tabular}[c]{@{}l@{}}0.3432 0.3203\\ \textbf{0.6635} 3579\end{tabular} & \begin{tabular}[c]{@{}l@{}}0.0530 0.1408\\ 0.1938 \textbf{0}\end{tabular}   \\ \hline
DROP~\cite{glocker2011deformable}            & \begin{tabular}[c]{@{}l@{}}2.0554 0.3084\\ 2.3638 2515\end{tabular}& \begin{tabular}[c]{@{}l@{}}0.2509 \textbf{0.1518}\\ 0.4027 \textbf{0}\end{tabular}    & \begin{tabular}[c]{@{}l@{}}0.1788 0.1634\\ 0.3423 \textbf{0}\end{tabular}  & \begin{tabular}[c]{@{}l@{}}0.3021 0.1186\\ 0.4207 \textbf{0}\end{tabular}    & \begin{tabular}[c]{@{}l@{}}1.2514 \textbf{0.1406}\\ 1.3920 \textbf{0}\end{tabular}    & \begin{tabular}[c]{@{}l@{}}0.5065 0.6193\\ 1.1259 3663\end{tabular} & \begin{tabular}[c]{@{}l@{}}0.2415 0.1841\\ 0.4256 \textbf{0}\end{tabular}   \\ \hline
\end{tabular}
}
\caption{The performance of different image registration methods for various synthetic and real medical images. Here, $E_{\text{sim}}$ measures the accuracy of the registration mapping as described in Equation~\eqref{eqt:esim}, $E_{\text{smooth}}$ measures the smoothness of the mapping as described in Equation~\eqref{eqt:esmooth}, $E_{\text{total}}$ measures the overall quality of the mapping as described in Equation~\eqref{eqt:etotal}, and the number of flips reflects the bijectivity of the mapping. For each example and each measure, the best entry among all methods is in bold.}
\label{table: comparison}
\end{table}

%%%%%%%%%%%%%%%%%%%%%%%%%%%%%%%%%%%%%%%%%%%%%%%%%%%%%%%%%%%%%
\section{Conclusion and future works}\label{sect:conclusion}
In this work, we have proposed a novel method for large deformation image registration by taking the advantage of the mathematical properties of quasiconformal maps and the robustness of CNNs. With a novel fidelity term based on features extracted using a pre-trained classification CNN included in our proposed energy model, we are able to obtain meaningful descent directions for registering two images without imposing any landmark constraints. We also use quasiconformal theory for regularizing the registration process to ensure the bijectivity and reduce the local geometric distortion of the resulting mappings. Our experiments have shown that the proposed method outperforms the existing methods for the registration of both synthetic images and real medical images.

A natural next step is to consider improving the performance of the registration using a more data-specific network. Currently, the CNN features are extracted using a common pre-trained classification network so that the proposed model can handle various types of images well. However, in many applications, the users are only interested in registering a specific type of images. For instance, ophthalmologists may only perform image registration on retina scans; pulmonologists may focus on registering lung scans of the patients; and neurologists may be most interested in registering brain MRI scans. By pre-training a more data-specific network, we can better include the prior information of a specific type of images and hence further improve the performance of our registration method for them.

Another possible future direction is to extend the proposed method for other registration problems. Using conformal parameterization algorithms for meshes~\cite{choi2015fast,choi2018linear,choi2020parallelizable} or point clouds~\cite{meng2016tempo,liu2020free}, we can effectively flatten two 3D mesh- or point-based surfaces $\mathcal{M}, \mathcal{S}$ onto the 2D plane with low geometric distortion (denote the mappings by $\varphi_{\mathcal{M}}:\mathcal{M} \to \mathbb{C}$ and $\varphi_{\mathcal{S}}:\mathcal{S} \to \mathbb{C}$). Then, we can devise a similar energy minimization model for obtaining a quasiconformal map $f$ between the flattened surfaces without imposing any landmark constraints. The composition map $\varphi_{\mathcal{S}}^{-1} \circ f \circ \varphi_{\mathcal{M}}$ will then give a meaningful registration between the two surfaces $\mathcal{M}, \mathcal{S}$. As manual landmark labeling is not needed, such a surface registration method may largely facilitate 3D shape morphometry~\cite{choi2020tooth}. With the increasing importance of 3D image registration, it will also be interesting to consider combining 3D quasi-conformal maps~\cite{lee2016landmark,zhang20203d} and CNN features. These ideas will be further explored and validated in our future work.

\bibliographystyle{AIMS.bst}
\bibliography{mybib.bib}

\end{document}